\documentclass{colt2015}

\usepackage{times}
\usepackage{amsmath}
\usepackage{xspace}

\usepackage{color-edits}
\addauthor{as}{red}      

\newcommand{\SparringEXP}{\ensuremath{\mathtt{SparringEXP4.P}}\xspace}
\newcommand{\SparringFPL}{\ensuremath{\mathtt{SparringFPL}}\xspace}
\newcommand{\FPL}{\ensuremath{\mathtt{FPL}}\xspace}
\newcommand{\ProjectedGD}{\ensuremath{\mathtt{ProjectedGD}}\xspace}
\newcommand{\expfp}{\ensuremath{\mathtt{Exp4.P}}\xspace}

\newcommand{\eps}{\ensuremath{\varepsilon}}

\newcommand{\indicator}[1]{\mathbf{1}\{#1\}}
\newcommand{\1}{\indicator}

\newcommand{\argmin}{\operatornamewithlimits{argmin}}
\newcommand{\xhdr}[1]{\vspace{1mm} \noindent{\bf #1}}

\renewcommand{\eqref}[1]{Eq.~(\ref{#1})}
\newcommand{\eqrefii}[2]{Eqs.~(\ref{#1}) and~(\ref{#2})}

\newcommand{\GA}{{G_A}}
\newcommand{\Gpi}{{G_{\pi}}}

\newcommand{\exprow}{{row-Exp}}
\newcommand{\expcol}{{column-Exp}}
\newcommand{\rowfpl}{{row-FPL}}
\newcommand{\colfpl}{{column-FPL}}

\newcommand{\pref}{\mathbf{P}}
\newcommand{\preft}{\pref_t}
\newcommand{\prefi}{\pref_i}
\newcommand{\prefelt}[2]{{P({#1},{#2})}}
\newcommand{\prefeltt}[2]{{P_t({#1},{#2})}}
\newcommand{\prefelti}[2]{{P_i({#1},{#2})}}
\newcommand{\prefab}{\prefelt{a}{b}}
\newcommand{\preftab}{\prefeltt{a}{b}}
\newcommand{\prefiab}{\prefelti{a}{b}}
\newcommand{\prefaa}{\prefelt{a}{a}}

\newcommand{\epref}{\mathbf{\hat{P}}}
\newcommand{\eprefi}{{\epref_i}}
\newcommand{\eprefielt}[2]{{\hat{P}_i({#1},{#2})}}
\newcommand{\eprefiaibi}{\eprefielt{a_i}{b_i}}
\newcommand{\eprefiab}{\eprefielt{a}{b}}

\newcommand{\eprefbnd}{L}
\newcommand{\eprefvar}{V}
\newcommand{\empmatbnd}{\varepsilon'}
\newcommand{\algerr}{\varepsilon}

\newcommand{\Msup}{\mathbf{M}}
\newcommand{\Msupelt}[2]{{M({#1},{#2})}}
\newcommand{\altpi}{\rho}
\newcommand{\Msuppipi}{\Msupelt{\pi}{\altpi}}

\newcommand{\Memp}{\mathbf{\hat{M}}}
\newcommand{\Mempelt}[2]{{\hat{M}({#1},{#2})}}
\newcommand{\Memppipi}{\Mempelt{\pi}{\altpi}}

\newcommand{\bldiag}{\mathbf{B}}
\newcommand{\bldiagelt}[4]{{B(({#1},{#2}),({#3},{#4}))}}

\newcommand{\numbatch}{m}

\newcommand{\astar}{{a^*}}

\newcommand{\X}{{\cal X}}
\newcommand{\Dpair}{{\cal D}}

\newcommand{\picard}{{|\Pi|}}

\newfont{\cmmib}{cmmib10}
\newcommand{\boldell}{{\mbox{\cmmib \symbol{'140}}}}

\newcommand{\decsp}{{\cal D}}
\newcommand{\losssp}{{\cal L}}
\newcommand{\dd}{\mathbf{d}}
\newcommand{\loss}{\boldell}

\newcommand{\perturb}{\mathbf{p}}
\newcommand{\perturbalt}{\mathbf{q}}
\newcommand{\reals}{{\mathbb{R}}}

\newcommand{\fplparm}{{\alpha}}
\newcommand{\lone}[1]{{\|{#1}\|_1}}
\newcommand{\Exp}[2]{{\mathbb{E}_{#1}\left[{#2}\right]}}
\newcommand{\expect}[2]{{\mathbb{E}_{#1}[{#2}]}}

\newcommand{\given}{\mathbin{\vert}}

\newcommand{\trans}[1]{{#1^{\top}}}

\newcommand{\prsim}[1]{{\Delta_{#1}}}

\newcommand{\uu}{\mathbf{u}}
\newcommand{\ww}{\mathbf{w}}

\newcommand{\wpolhat}{\mathbf{\hat{W}}}
\newcommand{\wwpol}{\mathbf{W}}
\newcommand{\uupol}{\mathbf{U}}
\newcommand{\wpol}{W}
\newcommand{\upol}{U}
\newcommand{\zz}{\mathbf{z}}
\renewcommand{\ss}{\mathbf{s}}
\newcommand{\vv}{\mathbf{v}}
\newcommand{\vpi}{\vv_\pi}
\newcommand{\valtpi}{{\vv_\altpi}}
\newcommand{\vpielt}[2]{{v_\pi({#1}, {#2})}}
\newcommand{\C}{{\cal C}}
\newcommand{\ubar}{{\overline{\uu}}}
\newcommand{\wwbar}{{\overline{\ww}}}
\newcommand{\vbar}{{\overline{\vv}}}
\newcommand{\sbar}{{\overline{\ss}}}
\newcommand{\vstar}{{{\vv}^*}}

\newcommand{\regret}{{\algerr}}
\newcommand{\regin}{{\delta}}

\newcommand{\Tout}{{N_{out}}}
\newcommand{\Tin}{{N_{in}}}
\newcommand{\minmaxbudget}{{N}}

\newcommand{\brackets}[1]{{\left[{#1}\right]}}
\newcommand{\parens}[1]{{\left({#1}\right)}}
\newcommand{\paren}{\parens}
\newcommand{\abs}[1]{{\left|{#1}\right|}}
\newcommand{\ceiling}[1]{{\left\lceil{#1}\right\rceil}}

\newcommand{\lensq}[1]{{\|{#1}\|^2}}

\newcommand{\len}[1]{\|{#1}\|}

\newcommand{\approxproject}{{\ensuremath{\mathtt{ApproxProject}}\xspace}}
\newcommand{\approxfactor}{{\alpha}}

\newcommand{\costvec}{\mathbf{c}}
\newcommand{\cost}{{c}}

\newcommand{\sectref}[1]{Section~\ref{#1}}

\newcommand{\compactitemize}{\itemsep0pt \parskip0pt \parsep0pt}

\definecolor{MyBlue}{RGB}{95,166,180}

\makeatletter
\renewcommand*{\thanks}[1]{%
  \footnotemark
  \protected@xdef\@thanks{\@thanks
    \protect\footnotetext[\arabic{footnote}]{#1}}%
}
\makeatother

\title[Contextual Dueling Bandits]{Contextual Dueling Bandits}

 \coltauthor{\Name{Miroslav Dud\'{i}k} \Email{mdudik@microsoft.com}\\
 \addr Microsoft Research, New York, NY USA
 \AND
 \Name{Katja Hofmann} \Email{katja.hofmann@microsoft.com}\\
 \addr  Microsoft Research, Cambridge, United Kingdom
 \AND
 \Name{Robert E. Schapire}\thanks{On leave from Princeton University.} \Email{schapire@microsoft.com}\\
 \addr  Microsoft Research, New York, NY USA
 \AND
 \Name{Aleksandrs Slivkins} \Email{slivkins@microsoft.com}\\
 \addr  Microsoft Research, New York, NY USA
  \AND
 \Name{Masrour Zoghi}\thanks{Part of this research was conducted during an internship with Microsoft Research.} \Email{m.zoghi@uva.nl}\\
 \addr  University of Amsterdam, Amsterdam, The Netherlands
 }

\begin{document}

\maketitle

\begin{abstract}
We consider the problem of learning to choose actions using
contextual information when
provided with limited feedback in the form of relative pairwise
comparisons.
We study this problem in the dueling-bandits framework of
\citet{yue09:k-arm}, which
we extend to incorporate context.
Roughly, the learner's goal is to find the best policy, or way of
behaving, in some space of policies, although ``best'' is not always
so clearly defined.
Here, we propose a new and natural solution concept, rooted in game theory,
called a
\emph{von Neumann winner}, a randomized policy that beats or ties
every other policy.
We show that this notion overcomes important
limitations of existing solutions, particularly the Condorcet winner which
has typically been used in the past, but which requires strong and often
unrealistic assumptions.
We then present three {efficient} algorithms for online learning in our setting,
and for approximating a von Neumann winner from batch-like
data.
The first of these algorithms achieves particularly low regret, even
when data is adversarial, although its time and space requirements are
linear in the size of the policy space.
The other two algorithms require time and space only
logarithmic in the size of the policy space when provided access to an
oracle for solving classification problems on the space.
\end{abstract}

\begin{keywords}
contextual dueling bandits, online learning, bandit algorithms, game theory.
\end{keywords}

\section{Introduction}
\label{sec:intro}

We study how to learn to act
based on contextual
information when provided only with partial, relative feedback.
This problem naturally arises in information retrieval (IR) and recommender systems, where the user feedback is considerably more reliable when interpreted as 
relative comparisons rather than absolute labels
\citep{radlinski2008:how}.
For instance, in web search, for a particular query, the IR system may
have several candidate rankings of documents that could be presented,
with the best option being dependent upon the specific user.
By presenting a mix or interleaving of two of the candidate rankings and
observing the user's response \citep{chapelle12:large,hofmann13:fidelity},
it is possible for such a system to get feedback about user preferences.
However, this feedback is partial since it is only
with respect to the two rankings that were chosen, and it is relative
since it only tells which of the two rankings is
preferred to the other.

The {\em dueling-bandits problem} of \citet{yue09:k-arm} formalizes
this setting.
Abstractly,
the learner is repeatedly faced with a set of possible actions,
and may select two of these actions to face off in a {\em duel}
whose stochastically determined winner is then revealed.
Through such experimentation, the learner attempts to find the
``best'' of the actions.

In this paper,
we focus on the \emph{contextual dueling bandit} setting, where context can provide information that helps identify the best action.
For instance,
in the example above, the actions may be the candidate
rankings to choose among, and the context may be
additional information about the user or query that might help in
choosing the best ranking.
The learner's goal now is to find a good {\em policy}, a rule for choosing
actions based on context.

Similar to prior work
on contextual (non-dueling)
bandits~\citep{EXP4,LangfordZh07,dudik2011efficient,AgarwalEtAl14},
we propose a setting in which
the learner has access to a space of policies $\Pi$, with the goal of
performing as well as the ``best'' in the space.
This space plays a role analogous to the hypothesis space in
supervised learning.
It will typically be extremely large or even infinite.
We therefore explicitly aim for methods that will be applicable
when this is the case.

Merely defining the precise goal of learning can be problematic
in such a relative-feedback setting.
When rewards are absolute, the best policy in $\Pi$ is clearly and easily
defined as the one that achieves the highest expected reward, because, by
such an absolute measure, this policy beats every other policy.
In a relative-feedback setting, since we have a means of obtaining pairwise
comparisons between actions or policies,
we might aim to choose the policy
in $\Pi$ that (on average)
beats every other policy in the class in such head-to-head
competitions.
Most previous work on dueling bandits~\citep{yue09:k-arm,yue2011beat,Urvoy:2013,RUCB2014} has in fact explicitly or
implicitly assumed that such a {\em Condorcet winner\/}
exists.
But there are good reasons to doubt such a strong assumption,
particularly when working with large and rich policy spaces.
There are numerous examples, even in natural situations, where this
assumption (and more generally, transitivity among policies) is known
to fail (see, for instance, \citealp{gardner:1970,RCS2014}).
Indeed, the preferences of a population of users do not need to be
transitive, even if each individual user has transitive preferences.

In this paper, we seek to improve the dueling bandits techniques in two respects.
First, we seek to relax the modeling restrictions on which previous
methods have depended so as to develop methods that are more generally
applicable.
Second, we seek to achieve a similar level
of flexibility in the design of policies
as for supervised learning algorithms.

\xhdr{Contributions.}
Our first contribution (in \sectref{sec:vonneumann}) is the introduction of a new
solution concept, called the \emph{von Neumann winner}, which is based
on a game-theoretic interpretation.
Like a Condorcet winner, when facing any other policy in a duel, a von Neumann winner has at least a 50\% chance of winning; in this sense, a von Neumann winner is at least as good as every policy in the space.
On the other hand, a von Neumann winner is {\em always\/} guaranteed to
exist, without any extraneous assumptions.
This guarantee is made possible by allowing
policies to be selected in a {\em randomized\/} fashion, as is quite natural in such a
learning setting.

With the goal of learning clarified, we turn to algorithms.
As a warm-up, in \sectref{sec:exp4-spar}, we give a fully online algorithm in which
two copies of the \expfp\ multi-armed bandit algorithm~\citep{beygelzimer11:epoch}
are run against one another
(using a ``sparring'' approach previously suggested by~\citealp{Ailon:2014}).
Although yielding good regret,
this algorithm requires time and space linear in $\picard$, which is impractical
in most realistic settings where we would expect $\Pi$ to be enormous.

To address this difficulty, we propose an approach
used previously in other works on contextual bandits
\citep{LangfordZh07,dudik2011efficient,AgarwalEtAl14}.
Specifically, we assume that
we have access to a \emph{classification oracle} for our policy
class that can find the minimum-cost policy in $\Pi$ when given
the cost of each action on each of a sequence of contexts.
In fact, an ordinary cost-sensitive, multiclass classification
learning algorithm can be used for this purpose, which suggests that,
practically, this may be a reasonable and natural assumption.

We then consider techniques for constructing a von Neumann winner from
empirical exploration data.
(Although we focus on a batch-like setting, the resulting algorithms can be used online as well.)
We analyze the statistical efficiency of this approach in \sectref{sec:explore-first}.
In Sections \ref{sec:FPL} and \ref{sec:PGD}, we  give two polynomial-time algorithms for
computing an approximate von Neumann winner from data: one based on
\citeauthor{kalai03:fpl}'s Follow-the-Perturbed-Leader algorithm~\citeyearpar{kalai03:fpl}, and the other based
on projected gradient ascent as studied by ~\citet{Zinkevich03:online}.
These techniques yield learning algorithms that approximate or perform as well as the von Neumann winner, using data, time, and space that only depend
logarithmically on the cardinality of the space $\Pi$, and therefore,
are applicable even with huge policy spaces.


\xhdr{Other related work.}
Numerous algorithms have been proposed for the (non-contextual) dueling bandits problem: Interleaved Filter \citep{yue09:k-arm}; Beat the Mean (BTM) \citep{yue2011beat}; Sensitivity Analysis of VAriables for Generic Exploration (SAVAGE) \citep{Urvoy:2013}; Relative Confidence Sampling \citep{RCS2014}; Relative Upper Confidence Bound (RUCB) \citep{RUCB2014}; Doubler, MultiSBM and Sparring \citep{Ailon:2014} and mergeRUCB \citep{mergeRUCB2015}. These methods impose various constraints on the problem at hand, ranging from the requirement that it arise from an underlying utility function (e.g., MultiSBM) to no constraint at all (e.g., SAVAGE); they mainly provide regret bounds that are logarithmic in the number of rounds, and at least linear in the number of actions.
In principle, these methods could be applied to contextual dueling bandits by treating policies as actions.
But this would lead to regret at least linear in the number of policies which is far worse than the logarithmic bounds obtained in this paper.

The method that is the most closely related to our work is Dueling Bandit Gradient Descent (DBGD) \citep{Yue:2009}, a policy gradient method that iteratively improves upon the current policy by conducting comparisons with nearby policies, assuming that the policy space comes equipped with a distance metric, and incrementally adapting the policy if a better alternative is encountered.
As with all local optimization methods, DBGD imposes a convexity assumption on the dueling bandit problem for its performance guarantee: the dueling bandit problem is assumed to arise from the noisy observations of an underlying convex objective function. In this paper, we both relax the assumptions imposed by DBGD and improve upon the regret bound.

\section{Dueling bandits and the von Neumann winner}
\label{sec:vonneumann}

In the \emph{dueling bandits problem}~\citep{yue09:k-arm},
the learner has access to $K$ possible {\em actions},
$1,\ldots,K$, and attempts to determine the ``best'' action through
repeated stochastic pairwise comparisons of actions, called {\em duels}.
Thus, at each time step, the learner
chooses a pair of actions $(a,b)$ for a duel;
the outcome of the duel is $+1$ if $a$ wins, and $-1$ if $b$ wins.
The (unknown) expected value of this outcome is denoted $\prefab$, and is assumed
to depend only on the selected pair $(a,b)$.
In other words, the probability that $a$ beats $b$ in a duel is
$(\prefab + 1)/2$, and the two actions are exactly evenly matched if
$\prefab=0$.
We say that {\em $a$ beats $b$} to mean that the chance of
$a$ winning a duel with $b$ is strictly greater than $1/2$;
similarly, {\em $a$ ties $b$} if this probability is exactly $1/2$.

The $K\times K$ matrix $\pref$ of all such expectations $\prefab$ is called the
\emph{preference matrix}.\footnote{In the literature, the preference matrix $\pref$ often refers to the matrix of probabilities $(\prefab+1)/2$. With our modification, $\prefab$ becomes anti-symmetric around $0$, which simplifies the arguments considerably.}
This matrix is initially unknown to the learner, but can be discovered
bit-by-bit through experimentation.
We assume, of course, that all of the entries of $\pref$ are in
$[-1,+1]$, and furthermore, that $\pref$ is {\em skew-symmetric},
meaning that $\trans{\pref}=-\pref$ so that a duel
$(b,a)$ is equivalent to (the negation of) a duel $(a,b)$.
(This also implies $\prefaa=0$ for every action $a$, as is natural.)
Other than this, we strenuously avoid making any assumptions in the
current work about the matrix $\pref$.
For instance, we do {\em not\/} make any assumptions regarding
transitivity among the various actions.

In such a relative-feedback setting, the ``best'' action
is not always well-defined because there is no measure of the absolute
quality of actions.
Existing work typically assumes the existence of a \emph{Condorcet
winner} \citep{Urvoy:2013,RUCB2014}, that is, an action $\astar$ that
beats every other action $a\neq \astar$.
This is a very natural definition from a preference learning
perspective, since $\astar$ is indeed preferred to every other
action. However, it has been shown that dueling bandit problems
without Condorcet winners arise regularly in
practice~\citep{RCS2014}.%
\footnote{See also Appendix~\ref{sec:non-condorcet} for more compelling evidence that this is indeed the case.}

Although there is no guarantee of a {\em single\/} action beating all others,
the situation changes considerably if we simply allow actions
to be selected in a {\em randomized\/} fashion.
With this natural relaxation, the problem of non-existence entirely
vanishes.
Thus, the idea is to find a probability vector $\ww$ in $\prsim{K}$
(where $\prsim{K}$ is the simplex of vectors in $[0,1]^K$ whose entries
sum to~$1$)
such that
\begin{equation}  \label{eqn:v1}
\textstyle  \sum_{a=1}^K  w(a) \prefab \geq 0
\quad \text{for all actions $b$}.
\end{equation}
{In words, for every action $b$, if $a$ is selected randomly according
to distribution $\ww$, then the chance of beating $b$ in a %
duel is at least $1/2$.}
(Note that this property implies that the same will be true
if $b$ is itself selected in a randomized way.)
A distribution $\ww$ with this property is said to be
a {\em von Neumann winner} for the preference matrix $\pref$.

As the name reflects, this notion is intimately connected to a
game-theoretic interpretation.
Indeed, we can view preference matrix $\pref$ as describing a
{\em zero-sum matrix game}.
In such a game, the two players simultaneously choose
distributions (or {\em mixed strategies})
$\ww$ and $\uu$ over rows and columns, respectively,
yielding a gain to the row player of $\trans{\ww}\pref\uu$.
According to von Neumann's celebrated minmax theorem,
for any matrix $\pref$,
\[
  \max_{\ww\in\Delta_K} \min_{\uu\in\Delta_K}  \trans{\ww}\pref\uu
  =
  \min_{\uu\in\Delta_K} \max_{\ww\in\Delta_K}  \trans{\ww}\pref\uu,
\]
the common value being the {\em value\/} of the game $\pref$.
A {\em maxmin strategy} $\ww$ or a {\em minmax strategy} $\uu$
is one realizing the $\max$ or $\min$ on the left- or right-hand side of this equality, respectively.
Finding these strategies is called {\em solving the game}.

We have assumed that the matrix $\pref$ is itself skew-symmetric,
so the game it describes is a
{\em symmetric game}.
Such games are known to have value exactly equal to zero (see, for instance,
\citealp[Theorem~II.6.2]{owen1995game}).
Working through definitions,
this means that $\ww$ is a maxmin strategy if and only if
$
   \min_{\uu\in\prsim{K}}  \trans{\ww}\pref\uu \geq 0.
$
But this is exactly equivalent to \eqref{eqn:v1}. Therefore:

\begin{proposition}  \label{prop:1}
A probability vector $\ww\in \prsim{K}$ is a von Neumann winner for
preference matrix $\pref$ if and only if
it is a maxmin strategy for the game $\pref$.
Consequently, every preference matrix $\pref$ has a von Neumann winner.
\end{proposition}

Before continuing, we briefly mention some of the
other solution concepts that have been proposed to remedy
the potential non-existence of a Condorcet
winner~\citep{schulze11:monotonic}.
Two of these are the Borda winner, the action
that has the highest probability of winning a duel against a
uniformly random action;
and the Copeland winner, the action that wins the most
pairwise comparisons~\citep{Urvoy:2013}.
Both of these fail the \emph{independence of
  clones criterion} \citep{schulze11:monotonic}, meaning that adding
multiple identical copies of an action can change the Borda or
Copeland winner.
This criterion is particularly crucial in a dueling bandit setting
because a given policy class may contain many identical policies.
In contrast, the von Neumann winner performs at least as
well as any individual policy, and is thus unaffected by the presence or
absence of clones. See Appendix~\ref{sec:copeland} for a more detailed discussion.
Note that if there does happen to exist a Condorcet winner, then it
will also be the unique von Neumann winner.


\section{Incorporating context}
\label{sec:context}

Next, we consider how the preceding development can be extended to a
much more realistic setting in which the best way of acting may depend
on additional, observable information, called {\em context}.
Thus, prior to choosing actions, the learner is allowed
to observe some value $x$, the {context}, selected by Nature from
some unspecified space $\X$.
For instance, $x$ might be a feature-vector description of a web user.
In this setting, the preference matrix is no longer static;
rather, which actions are better than which others now varies and depends on
the context, which therefore must be taken
into account to fully optimize the choice of actions.

Formally, we assume that on every round $t$ of the learning
process, a context $x_t$ and preference matrix $\preft$ are chosen by Nature.
The context $x_t$ is revealed to the learner, but the preference
matrix $\preft$ remains hidden.
Based on $x_t$, the learner selects two actions $(a_t,b_t)$ for a
duel, whose outcome has expectation determined by the current (hidden)
preference matrix $\preft$ in the usual way.
Except where noted otherwise, in this paper, we always assume that
each pair $(x_t,\preft)$ is chosen at random according to independent
draws from some unknown joint distribution $\Dpair$.

The goal is to determine which action to select as
a function of the context.
Such a mapping $\pi$ from contexts $x$ to actions $a$ is called a
{\em policy}.
Typically, we are working with policies of a particular form, that is,
from some {\em policy space} $\Pi$.
For instance, this space might
represent the set of all decision trees.
For simplicity, we assume that $\Pi$ has finite cardinality.
However, we generally think of $\Pi$ as an
extremely large space, exponential in any reasonable measure of
complexity.

The notion of von Neumann winner (as well as other concepts, like
Condorcet winner) can be extended to incorporate context
essentially by reducing to the non-contextual setting.
We can regard each policy $\pi$ as a
``meta-action,'' and define a $\picard\times\picard$ preference matrix
$\Msup$ over these meta-actions.
Thus, the rows and columns of $\Msup$ are each
indexed by policies in $\Pi$,
and
\begin{equation}  \label{eq:f6}
  \Msuppipi = \Exp{(x,\pref)\sim\Dpair}{\prefelt{\pi(x)}{\altpi(x)}}.
\end{equation}
This quantity is the expected outcome
when a ``meta-duel'' is held between the two {\em policies} $\pi$ and $\altpi$,
whose stochastic outcome is determined by randomly selecting $(x,\pref)\sim\Dpair$,
and then holding an ordinary duel on $\pref$ between
the actions $\pi(x)$ and $\altpi(x)$.
This huge matrix $\Msup$ thus encodes the probability of any policy
beating any other policy in a duel.

We can now define von Neumann winner in the contextual setting
to be an (ordinary) von Neumann winner for the matrix $\Msup$
(regarded here as a kind of ``meta-preference-matrix'').
Thus, unraveling definitions, a (contextual) von Neumann winner
is a probability distribution $\wwpol$
over \emph{policies} such that for every opposing policy $\altpi$,
if $\pi$ is chosen at random from
$\wwpol$, then the probability that $\pi$ beats $\altpi$
in a duel is at least $1/2$.
That is, the {\em randomized\/} policy defined by $\wwpol$ beats or ties every policy in the space $\Pi$.
By Proposition~\ref{prop:1}
(applied to $\Msup$), such a von Neumann winner must exist.

For the rest of the paper, we study how to
compute (or approximate) contextual von Neumann winners.
Of course, because the space $\Pi$ and corresponding matrix $\Msup$
are both gigantic, this will present significant computational
challenges.

\section{Learning scenarios}
\label{sec:scenarios}

We consider two possible learning scenarios.

In the simpler of these, called {\em explore-then-exploit}, we suppose
that the learner is allowed to explore for some number of rounds
$\numbatch$
(where, as described above, on each round, the learner is presented
with a random context and permitted to run and observe the outcome of a duel
between a pair of actions of its choosing).
At the end of these $\numbatch$ rounds, the learner outputs a distribution
$\wpolhat$ over policies in $\Pi$.
The learner's goal is to produce $\wpolhat$ which is an
{\em $\varepsilon$-approximate von Neumann winner}, that is, for which
\[
  \min_{\uupol\in\prsim{\picard}} \trans{\wpolhat} \Msup \uupol \geq -\varepsilon
\]
for some small $\varepsilon>0$.
In other words, for all $\pi\in\Pi$, $\wpolhat$ should beat $\pi$ with
probability at least $1/2-\varepsilon/2$.
Naturally, $\numbatch$ should be ``reasonable'' as a function
$\varepsilon$.
This setting is almost like learning from a passively selected
batch of training examples, except that the learner has an active
role in selecting which actions to play in each duel.

In the alternative {\em full-explore-exploit\/} setting, learning
occurs in a fully online manner across $T$ rounds (in the manner
described earlier), with performance
measured using some notion of {\em regret}.
In this paper, where we are working with policies and changing
preference matrices,
we propose to define regret to be
\begin{equation} \label{eq:rd1}
  \max_{\pi\in\Pi}  \frac{1}{2}
       \sum_{t=1}^T \bigl[ \prefeltt{\pi(x_t)}{a_t} + \prefeltt{\pi(x_t)}{b_t} \bigr].
\end{equation}
If we can find an algorithm for which this regret is $o(T)$, then
eventually the algorithm selects actions $(a_t,b_t)$ which cannot be
beaten by any other policy $\pi\in\Pi$.

In the standard dueling-bandits setting with a static preference
matrix, a seemingly different definition of regret was used
by \citet{yue09:k-arm} in terms of an assumed Condorcet winner.
However, when specialized to their setting, and when provided with
their same assumptions, their definition can be shown to be equivalent
(up to constant factors) to \eqref{eq:rd1}.

\section{Sparring \expfp}
\label{sec:exp4-spar}

Our goal then is to find, approximate or perform as well as a
von Neumann winner, which, as we have seen, is a maxmin strategy for a
particular game.
Under this interpretation, it becomes especially natural to use
ordinary no-regret learning algorithms as players of this game since
it is known that such algorithms, when properly configured for this
purpose, will converge to maxmin or minmax
strategies~\citep{FreundSc-GEB}.
The idea is simply to run two independent copies of such an algorithm
against one another.
Such a ``sparring'' approach was previously proposed for dueling
bandits by \cite{Ailon:2014}, though without details, and not in the
contextual setting.

We consider using the multi-armed bandit algorithm \expfp~\citep{beygelzimer11:epoch}
for this purpose in the full-explore-exploit setting.
\expfp\ is well-suited since it is
designed to work with partial information as in our bandit setting,
and since it can handle the kind of adversarially generated data that
arises unavoidably when playing a game.
It also is designed to work with policies in a contextual setting like
ours (or, more generally, to accept the advice of ``experts'').

The learning setting for \expfp\ is as follows (somewhat,
but straightforwardly, modified for our present purposes).
There are $K$ possible actions, $1,\ldots,K$, and a finite space $\Pi$
of policies.
On each round $t=1,\ldots,T$, an adversary chooses and reveals a context
$x_t$, and also chooses, but does {\em not} reveal rewards
$r_t(1),\ldots,r_t(K)\in [-1,+1]$ for each of the $K$ actions.
The learner then selects an action $a_t$, and receives the revealed reward $r_t(a_t)$.
The learner's total reward is thus
  $\GA = \sum_{t=1}^T  r_t(a_t)$,
while the reward of each policy $\pi$ is
  $\Gpi = \sum_{t=1}^T  r_t(\pi(x_t))$.
The learner's goal is to receive reward close to that of the best
policy.
\cite{beygelzimer11:epoch} prove that (subject to very benign conditions)
with probability at least $1-\delta$,
\expfp\ achieves reward
\begin{equation}  \label{eq:exp4-g1}
    \GA \geq \max_{\pi\in\Pi}\; \Gpi - 12\sqrt{KT \ln(\picard/\delta)}.
\end{equation}
(This holds for any $\delta>0$; the $\delta$ is passed as a parameter to the algorithm.)

For contextual dueling bandits,
we run two separate copies of \expfp\ which are played against one another; let us call them
\exprow\ and \expcol.
We use the same actions, contexts, and policies for the two copies as for the original problem.
On each round $t$, Nature chooses a context $x_t$ and a preference
matrix $\preft$.
The context (but not the preference matrix) is revealed to \exprow\ and \expcol, which
select actions $a_t$ and $b_t$, respectively.
A duel is then held between these two actions; the outcome $r$ is
passed as feedback to \exprow\ (for its chosen action $a_t$),
and its negation $-r$ is similarly passed to \expcol. We call this algorithm \SparringEXP.

\begin{theorem}
\label{thm:sparring-exp4}
Consider $K$ actions, policy space $\Pi$, and time horizon $T$. Fix parameter $\delta>0$.
Then with probability at least $1-\delta$, \SparringEXP achieves regret at most
$
    O(\sqrt{KT \ln(\picard/\delta)})
$.
\end{theorem}

The proof is in Appendix~\ref{sec:pf-SparringEXP}. This result holds also for an adversarial environment in which the pairs $(x_t,\preft)$ are selected by an adversary rather than at random. Also, we can adapt this algorithm for explore-then-exploit learning using the following standard technique for online-to-batch conversion. Run \SparringEXP for $\numbatch$ exploration rounds. In each round $i$, \exprow\ internally computes a distribution $\ww_i$ over policies. Then $\wwbar = (1/\numbatch) \sum_{i=1}^\numbatch \ww_i$ is an \eps-approximate von Neumann winner where
$\eps=O(\sqrt{K\ln(\picard/\delta)/\numbatch})$.

Although yielding very good regret bounds and handling adversaries, this approach requires
time and space proportional to $\picard$, and is therefore not
practical for extremely large policy spaces.

\section{Explore-then-exploit algorithms with a classification oracle}
\label{sec:explore-first}

We next begin a development that will lead to efficient methods
(in terms of time, space and data) for handling
even extremely large policy spaces, under a particular assumption
discussed below.
We describe a general approach for exploration, for using the
collected data to find a statistically sound solution, and for
reducing the problem that must be solved to a more
tractable form.

We focus mainly on the explore-then-exploit problem.
Thus, we have $\numbatch$ exploration rounds, and on each round $i$, a pair $(x_i,\prefi)$ is selected
at random, and the learner is permitted to choose and observe the
outcome of a single duel $(a_i,b_i)$.
Although $x_i$ is observed, $\prefi$ is not.
Here, we propose a simple exploration strategy, called \emph{uniform exploration}, in which
each dueling pair $(a_i, b_i)$ is selected uniformly at random.
Let $r_i$ be the resulting observed outcome.
Based on these, the learner can obtain a noisy but unbiased
version of the hidden preference matrix $\prefi$.
Specifically, let us define a matrix $\eprefi$ where $\eprefiaibi=K^2 r_i$,
and all other entries $\eprefiab$ are set to zero.
It can be verified that the expected value of each entry
$\eprefiab$ is exactly $\prefiab$; that is,
$\expect{}{\eprefi \given \prefi} = \prefi$.

In Appendix~\ref{appendix-proof-perceptron}, we extend our setting to an arbitrary unbiased estimator $\eprefi$ of $\prefi$, and in particular to an arbitrary exploration strategy that does not change adaptively over time. This extension is parameterized by upper bounds on the absolute value and the variance of $\eprefiab$ for all rounds $i$ and all actions $(a,b)$. For uniform exploration, both upper bounds are $K^2$.


While non-adaptive exploration strategies usually lead to suboptimal statistical performance, they are often preferable in practice. This is because in large-scale industrial applications the existing infrastructure is often insufficient to support a feedback loop that would update the exploration strategy adaptively over time, and upgrading the infrastructure may be infeasible in the near term.

\xhdr{Statistical guarantees.}
With these noisy versions of the empirical preference matrices, we can
estimate the expected outcome in a ``meta-duel'' between two policies
$\pi$ and $\altpi$, that is, an entry of the matrix $\Msup$ defined
in \eqref{eq:f6}.
In particular, let
\begin{equation} \label{eq:f3}
 \Memppipi = \frac{1}{\numbatch}
             \sum_{i=1}^{\numbatch} \eprefielt{\pi(x_i)}{\altpi(x_i)}.
\end{equation}
Then the expected value of this quantity is $\Msuppipi$, the
corresponding entry of $\Msup$.
Moreover, using Bernstein's inequality and the union bound, we can
show that, with probability at least $1-\delta$,
\begin{equation}  \label{eq:f2}
   \abs{\Memppipi - \Msuppipi} \leq \empmatbnd
   \;\;\;\;\mbox{for all $(\pi,\altpi)\in\Pi\times\Pi$,}
\end{equation}
where
$\empmatbnd =
    O\bigl(K
        \sqrt{{\ln(\picard/\delta)}/{\numbatch}}\bigr)
$.
\footnote{The proof for \eqref{eq:f2} and the subsequent Lemma~\ref{lem:explore-first} are in Appendix~\ref{sec:pf-ExploreFirst}.}
Thus, although huge, the matrix $\Msup$ is well-approximated by the
matrix $\Memp$ using only a moderately sized sample.
In fact, to find an approximate maxmin strategy for $\Msup$ it suffices to find one for $\Memp$,
which will be the approach taken by our algorithms.


\begin{lemma}
\label{lem:explore-first}
Given the set-up above, suppose that \eqref{eq:f2} holds
(as will be the case with probability at least
$1-\delta$), and suppose further that
$\wpolhat\in\prsim{\picard}$ is a probability vector for which
\[
\min_{\uupol\in\prsim{\picard}} \trans{\wpolhat} \Memp \uupol
 \geq
\max_{\wwpol\in\prsim{\picard}} \min_{\uupol\in\prsim{\picard}} \trans{\wwpol} \Memp \uupol
- \algerr.
\]
Then $\wpolhat$ is a $(2\empmatbnd + \algerr)$-approximate von Neumann
winner for $\Msup$.
\end{lemma}


\xhdr{A more compact version of the problem.}
Our aim now is to find an approximate maxmin strategy for the
matrix $\Memp$.
Although this matrix is gigantic in both dimensions, by leveraging how
it was constructed from only a small number of empirical observations,
we can re-express the problem in a far more compact form.
To this end, let us define, for each policy $\pi\in\Pi$,
a {\em policy vector} $\vpi\in\reals^{\numbatch K}$ that encodes the behavior of
$\pi$ on the exploration data.
For readability, although a vector, we index entries of $\vpi$ by
pairs $(i,a)$, where $i$ is a round and $a$ is an action,
and we define
\[
   \vpielt{i}{a} = \1{\pi(x_i) = a}/\sqrt{\numbatch}.
\]
Thus, $\vpi$ is broken into $\numbatch$ length-$K$ blocks, with
block $i$ encoding in a natural way the action selected by
$\pi$ on $x_i$.
(The constant $1/\sqrt{\numbatch}$ is for normalization.)

We also define an
$\numbatch K \times \numbatch K$
block-diagonal matrix $\bldiag$, where the
$\numbatch$ blocks along the diagonal are exactly the
$K\times K$ matrices
$\eprefi$ described above.
Formally, using the earlier indexing,
\[
 \bldiagelt{i}{a}{j}{b} = \1{i=j} \eprefielt{a}{b}.
\]

Working through these definitions, it can be verified that for any two
policies $\pi$ and $\altpi$, the quantity
$\trans{\vpi} \bldiag \valtpi$
is exactly equal to $\Memppipi$ as defined in \eqref{eq:f3}.
This means that if $\wwpol$ and $\uupol$ are probability vectors over $\Pi$, then
\[
  \trans{\wwpol} \Memp \uupol
   = \textstyle \trans{\parens{\sum_{\pi\in \Pi} \wpol(\pi)\, \vpi}}\; \bldiag\;\;
            {\parens{\sum_{\altpi\in \Pi} \upol(\altpi)\, \valtpi}}.
\]
Therefore, the problem of finding a maxmin strategy for $\Memp$ is
equivalent to solving

\vspace{-4mm}

\begin{equation} \label{eq:f4}
  \max_{\ww\in\C} \min_{\uu\in\C} \trans{\ww} \bldiag \uu
\end{equation}
where $\C$ is the convex hull of the set of all policy vectors
$\{\vpi : \pi\in\Pi\}$ (henceforth, the \emph{policy hull}).
Furthermore, a solution $\ww\in\C$ is necessarily a convex combination of
vectors $\vpi$, and therefore corresponds to a probability vector over
policies.

The formulation given in \eqref{eq:f4} shows that $\bldiag$ should
itself be viewed as a game matrix, and that our remaining goal is to
approximately solve this game.
This matrix has the advantage of being far smaller than $\Memp$.
However, unlike a conventional matrix game, the space from which the
players' vectors $\ww$ and $\uu$ are chosen is not the standard space
of probability vectors over actions, but rather the convex hull of an exponentially
large set of vectors.

\xhdr{Classification oracle.}
Our algorithms assume that the policy space $\Pi$ is structured
in a way that admits a certain computational operation
that is quite natural in the realm of learning.
Specifically, we assume the existence of a {\em classification oracle}.
The input to this oracle is a sequence of {\em cost vectors}
$\costvec_1,\ldots,\costvec_\numbatch$, each in $\reals^K$,
with the interpretation that $\cost_i(a)$ is the cost of choosing
action $a$ on context $x_i$.
The output of the oracle is the policy in $\Pi$ with minimum cost,
that is,
\begin{equation} \label{eq:f5}
  \argmin_{\pi\in\Pi} 
          \sum_{i=1}^{\numbatch}\; \cost_i(\pi(x_i)).
\end{equation}
Indeed, regarding the $x_i$'s as examples, the actions $a$ as
labels or classes, and the policies $\pi$ as classifiers, we see that
this oracle is in fact solving an empirical, cost-sensitive,
multi-class classification problem.
Thus, the assumption of such an oracle is an idealization based on the
numerous cases in which effective classification algorithms already
exist. In practice, the policy space $\Pi$ is usually defined as the space of all possible policies returned by a given classification algorithm,
and we hope that our methods will be effective when
using ordinary off-the-shelf classification algorithms as
oracle.

Equivalently, the classification oracle can be described in terms of
policy vectors.
Specifically, the cost vectors above can be identified
with their concatenation, a single vector $\costvec\in\reals^{\numbatch K}$,
divided naturally into $\numbatch$ blocks.
Then the problem given in \eqref{eq:f5} is the same as
\[
  \argmin_{\ww\in\C} \costvec\cdot\ww,
\]
where the $\argmin$ is over the policy hull $\C$ defined above. This is because the minimum, without loss of generality, will be a policy vector $\vpi$, where $\pi$ minimizes \eqref{eq:f5}.
Therefore, in what follows, we use expressions of this latter form
to indicate an invocation of our assumed classification oracle.

\xhdr{Algorithms and end-to-end guarantees.}
We design algorithms that compute an approximate von Neumann winner $\wpolhat$ by solving the optimization problem in \eqref{eq:f4}. Although there exist many methods for solving such a game, the challenge here is the requirement that the solution be in the policy hull $\C$. As already seen in \sectref{sec:exp4-spar}, regret minimization algorithms are a natural choice. However, most standard algorithms will not conform to this constraint. In the sections that follow, we provide two algorithms: Algorithm \SparringFPL that builds on the Follow-the-Perturbed-Leader algorithm of \citet{kalai03:fpl}, and Algorithm \ProjectedGD that builds on online projected gradient descent methods of \citet{Zinkevich03:online}.

\newcommand{\NCalls}{N}  
\newcommand{\PLength}{b} 

For a given approximation quality, the performance of either algorithm is characterized by several quantities: the sufficient number of exploration rounds, the running time, the storage requirement, and the number of policies in the support of $\wpolhat$. As it turns out, the key quantities are the number of exploration rounds and the number of oracle calls. We assume each oracle call returns both a policy vector and a corresponding policy, each representable using $\PLength$ bits.%
\footnote{Often, policies are specified as a parameter vector to some algorithm that implements them.
For finite classes, it is usually the case that $\PLength$ is roughly $O(\ln \picard)$.%
}
The solution $\wpolhat$ is specified by explicitly listing the probabilities and policies in its support.

\begin{theorem}\label{thm:main-eps}
Consider $K$ actions and a policy class $\Pi$ with $\PLength$-bit representation. Fix parameters $\eps,\delta>0$. Both \SparringFPL and \ProjectedGD compute an \eps-approximate von Neumann winner $\wpolhat$ with probability $1-\delta$ using
    $\numbatch= O((K^2/\eps^2)\ln(\picard/\delta))$
exploration rounds with uniform exploration strategy.
The number of oracle calls is
    $\NCalls = O((K^6/\eps^4)\ln(\picard/\delta))$   
for \SparringFPL and
    $\NCalls = O(K^8/\eps^4)$
for \ProjectedGD.
For both algorithms, disregarding oracle calls,
the running time is $O(\numbatch K\NCalls)$,
the storage requirement is $O(\PLength\NCalls)$, and
$\wpolhat$ is a distribution over at most $\NCalls$ policies.
 \end{theorem}

While \SparringFPL is very simple and intuitive, \ProjectedGD achieves a better number of oracle calls whenever
    $K \ll \sqrt{\ln(\picard/\delta)}$.

Our algorithms can be used in the full-explore-exploit setting as well: after $\numbatch$ exploration rounds with uniform exploration strategy, $\wpolhat$ is computed and used in the remaining rounds for both actions. The parameter $\numbatch$ is chosen in advance as a function of the time horizon $T$. The statistical performance is expressed via regret, as defined in \eqref{eq:rd1}. The total running time is dominated by the time to compute $\wpolhat$.%
\footnote{The running time to \emph{execute} $\wpolhat$ in each of the exploitation rounds (i.e., to compute the random action for a given context) is a low-order term; we omit further details from this version.}

\begin{theorem}[regret]\label{thm:main-regret}
Consider $K$ actions, time horizon $T\geq K$, and a policy class $\Pi$ with $\PLength$-bit representation. Fix a parameter $\delta>0$. Both \SparringFPL and \ProjectedGD achieve regret
    $O(K^{2/3}\, T^{2/3}\, \Psi^{1/3})$     
with probability $1-\delta$, where
    $\Psi = \ln(\picard/\delta)$.
The number of oracle calls is                         
    $\NCalls = O(K^{10/3}\, T^{4/3}\,\Psi^{-1/3})$   
for \SparringFPL and
    $\NCalls = O(K^{16/3}\, T^{4/3}\,\Psi^{-4/3})$   
for \ProjectedGD.
For both algorithms, disregarding oracle calls,
the total running time is $O(\numbatch K\NCalls)$,
the storage requirement is $O(\PLength\NCalls)$, and
the number of exploration rounds is
    $\numbatch = O(K^{2/3}\, T^{2/3}\,\Psi^{1/3})$.
\end{theorem}


\section{Solving the compact game with \SparringFPL}
\label{sec:FPL}

Our first algorithm to solve \eqref{eq:f4}, \SparringFPL, is based on the Follow-the-Perturbed-Leader (\FPL)
algorithm of \citet{kalai03:fpl}. \FPL is designed for a standard
online learning problem:
Let $\decsp$ and $\losssp$ be subsets of $\reals^{\numbatch K}$.
On each round $t=1,\ldots,\minmaxbudget$,
the learner chooses a decision vector $\dd_t\in\decsp$, and
then receives a loss vector $\loss_t\in\losssp$.
The learner's goal is to minimize its cumulative loss
$\sum_{t=1}^\minmaxbudget \dd_t\cdot\loss_t$ relative to the best
possible loss using a fixed decision, that is,
$\min_{\dd\in\decsp} \sum_{t=1}^\minmaxbudget \dd\cdot\loss_t$.
\FPL chooses $\dd_t$ as the best such vector based on a slightly
perturbed version of the preceding losses. Namely,
letting $\perturb_t\in\reals^{\numbatch K}$ be chosen uniformly at random
from $[0,1/\fplparm]^{\numbatch K}$,
\[ \dd_t = \argmin_{\dd\in\decsp} \dd\cdot
\textstyle            \paren{\sum_{\tau=1}^{t-1} \loss_{\tau} + \perturb_t}.
\]


We solve \eqref{eq:f4} by sparring two copies of
\FPL, called \rowfpl\ and \colfpl, in the fashion of a repeated game.
On every round $t$, \rowfpl\ uses \FPL to select a vector $\ww_t$,
while \colfpl\ uses a different copy of \FPL to select a vector
$\uu_t$.
We then define the resulting loss vectors to be $-\bldiag \uu_t$ for
\rowfpl, and $\trans{\bldiag}\ww_t$ for \colfpl.
Here is the complete algorithm:
\begin{itemize}  \compactitemize
\item
For $t=1,\ldots,\minmaxbudget$:
\begin{itemize}  \compactitemize
\item
Choose uniform random perturbations $\perturb_t$, $\perturbalt_t$ from
$[0,1/\fplparm]^{\numbatch K}$.
\item
Let $\ww_t = \argmin_{\ww\in\C}
           {\ww} \cdot \brackets{
                   -\bldiag (\uu_1 + \cdots + \uu_{t-1}) + \perturb_t
                                }$.

\item
Let $\uu_t = \argmin_{\uu\in\C}
         {\uu} \cdot \brackets{ \trans{\bldiag}(\ww_1 + \cdots + \ww_{t-1}) +
           \perturbalt_t }$.
\end{itemize}
\item
Output
$\wwbar = \frac{1}{\minmaxbudget}
          \sum_{t=1}^\minmaxbudget \ww_t$
\end{itemize}
The $\argmin$ expressions in the algorithm are implemented using the classification oracle.
The returned vector $\wwbar$ is in $\C$, and in fact corresponds
to a uniform mixture of $\minmaxbudget$ policies.

In Appendix~\ref{sec:pf-FPL}, we show that to find an \eps-approximate solution to \eqref{eq:f4} with probability $1-\delta$, it suffices to
use
    $\minmaxbudget = O(K^4/\eps^2) (\numbatch+\ln (1/\delta))$
steps of the algorithm
with
    $\fplparm=\sqrt{2/(K^4 \minmaxbudget)}$,
which in turn implies Theorems~\ref{thm:main-eps} and \ref{thm:main-regret} for \SparringFPL.

%
%

\section{Solving the compact game with \ProjectedGD}
\label{sec:PGD}

Our second algorithm, called \ProjectedGD, solves \eqref{eq:f4} using online projected gradient descent methods as studied by~\cite{Zinkevich03:online}. The algorithm maintains a vector $\ww_t\in\C$ corresponding to a
strategy for the row player.
On every round, a column strategy $\uu_t\in\C$ is chosen that is a ``best
response'' to $\ww_t$.
The strategy $\ww_t$ is updated by taking a small gradient step.
The resulting vector $\zz_{t+1}$ is likely to be outside the set
$\C$, and therefore is (approximately) projected back to $\C$,
yielding $\ww_{t+1}$.
The algorithm is as follows:
\begin{itemize}   \compactitemize
\item
Choose any $\ww_1\in\C$
\item
For $t=1,\ldots,\Tout$:

\begin{itemize}   \compactitemize
\item
$\uu_t = \arg\min_{\uu\in\C} \trans{\ww_t} \bldiag \uu$
\item
$\zz_{t+1} = \ww_t + \eta \bldiag \uu_t$
\item
$\ww_{t+1} = \approxproject(\zz_{t+1}, \ww_t)$
\end{itemize}

\item
Output $\wwbar=\frac{1}{\Tout} \sum_{t=1}^{\Tout} \ww_t$
\end{itemize}


Ideally, we would like for $\ww_{t+1}$ to be the exact Euclidean
projection of $\zz_{t+1}$ onto $\C$, but instead need to settle for an approximation.
For this purpose,
the procedure $\approxproject(\zz,\vv_1)$, described below,
computes an approximate
projection of an arbitrary vector $\zz$ onto $\C$.
It takes as an input a second vector $\vv_1$ that is already in $\C$, and which
we can think of as an initial guess at the actual projection.
The quality (as an approximation) of the returned vector $\vbar$ is
allowed to depend on how close $\vv_1$ is to $\zz$.
Specifically, we require that, for all $\ss\in\C$, and a constant $\approxfactor$ specified later,
\begin{equation} \label{eq:b1}
  \lensq{\ss - \vbar} \leq  \lensq{\ss - \zz}
            + \approxfactor \cdot \len{\vv_1 - \zz}.
\end{equation}

In Appendix~\ref{sec:pf-GD}, we show that with the parameter
    $\eta = {2}/({\eprefbnd\sqrt{\Tout}})$
our algorithm finds an \eps-approximate solution to \eqref{eq:f4}, where
    $\algerr = {2\eprefbnd}/{\sqrt{\Tout}}
               + {\eprefbnd \approxfactor}/{2}$.

%
%
%

\xhdr{Computing approximate projections.}
It remains to describe the approximate-projection procedure
$\approxproject(\zz,\vv_1)$.
Given an arbitrary vector $\zz$ and another
vector $\vv_1\in\C$,
the goal of the algorithm, as in \eqref{eq:b1}, can be restated as
that of finding a vector $\vbar\in\C$ for which
\begin{equation} \label{eq:g3}
  \min_{\ss\in\C} F(\ss,\vbar) \geq - \approxfactor \cdot \len{\vv_1 - \zz}
\end{equation}
where we define
$F(\ss,\vv)
 = \lensq{\ss - \zz} - \lensq{\ss - \vv}
 = 2\ss\cdot (\vv - \zz) + \lensq{\zz} - \lensq{\vv}$.
Note that $F$ is linear in $\ss$ (for each $\vv$), and concave in
$\vv$ (for each $\ss$).
To ensure that \eqref{eq:g3} holds, we give an algorithm that aims to
maximize the left-hand side of this inequality.
(As a side note, the maximizing vector turns out to be exactly the
projection of $\zz$ onto $\C$, although we do not require that fact
for our algorithm and analysis.)

To this end, we use an algorithm that resembles repeated play of
a game in which the payoff is defined by $F$.
The $\ss$ player uses best response on each round, while the $\vv$
player again uses a variant of online gradient ascent
applied to the function $F(\ss_t,\cdot)$.
The algorithm takes a parameter $\nu\in (0,1]$, and uses
$\vv_1\in\C$, which was provided as an argument
to $\approxproject(\zz,\vv_1)$, as the initial vector.
Here is the algorithm:

\begin{itemize}  \compactitemize
\item
For $t=1,\ldots,\Tin$:

\begin{itemize}  \compactitemize
\item
$\ss_t = \arg\min_{\ss\in\C} \ss\cdot (\vv_t - \zz)$
\item
$\vv_{t+1} = (1-\nu) \vv_t + \nu \ss_t$
\end{itemize}

\item
Output $\vbar=\frac{1}{\Tin} \sum_t \vv_t$
\end{itemize}
Note that $\vv_t$ is in $\C$ for every $t$ (by convexity of $\C$), and
therefore $\vbar$ is as well.

In Appendix~\ref{sec:pf-GD-inner}, we show that $\approxproject(\zz,\vv_1)$ with parameter
    $ \nu = {\len{\zz-\vv_1}}/{\sqrt{\Tin}}$
computes $\vbar$ that satisfies \eqref{eq:g3} with
    $\approxfactor = 8/\sqrt{\Tin}$.
We optimize the choice of $\Tin$ and $\Tout$ to show that one can obtain an $\algerr$-approximate solution to \eqref{eq:f4} using only
    $O(K^8 / \algerr^4)$
oracle calls. This in turn implies Theorems~\ref{thm:main-eps} and \ref{thm:main-regret} for \ProjectedGD.

\acks{We would like to thank Alekh Agarwal for insightful comments and discussions.}

\bibliography{bibliography,cdb}

\begin{thebibliography}{29}
\providecommand{\natexlab}[1]{#1}
\providecommand{\url}[1]{\texttt{#1}}
\expandafter\ifx\csname urlstyle\endcsname\relax
  \providecommand{\doi}[1]{doi: #1}\else
  \providecommand{\doi}{doi: \begingroup \urlstyle{rm}\Url}\fi

\bibitem[Agarwal et~al.(2014)Agarwal, Hsu, Kale, Langford, Li, and
  Schapire]{AgarwalEtAl14}
Alekh Agarwal, Daniel Hsu, Satyen Kale, John Langford, Lihong Li, and Robert~E.
  Schapire.
\newblock Taming the monster: {A} fast and simple algorithm for contextual
  bandits.
\newblock In \emph{Proceedings of the International Conference on Machine
  Learning (ICML)}, 2014.

\bibitem[Ailon et~al.(2014)Ailon, Karnin, and Joachims]{Ailon:2014}
Nir Ailon, Zohar Karnin, and Thorsten Joachims.
\newblock Reducing dueling bandits to cardinal bandits.
\newblock In \emph{Proceedings of the International Conference on Machine
  Learning (ICML)}, pages 856--864, 2014.

\bibitem[Auer et~al.(2002)Auer, Cesa-Bianchi, Freund, and Schapire]{EXP4}
Peter Auer, Nicol{\`o} Cesa-Bianchi, Yoav Freund, and Robert~E. Schapire.
\newblock The nonstochastic multiarmed bandit problem.
\newblock \emph{SIAM J.~Computing}, 32\penalty0 (1):\penalty0 48--77, 2002.

\bibitem[Beygelzimer et~al.(2011)Beygelzimer, Langford, Li, Reyzin, and
  Schapire]{beygelzimer11:epoch}
Alina Beygelzimer, John Langford, Lihong Li, Lev Reyzin, and Robert Schapire.
\newblock Contextual bandit algorithms with supervised learning guarantees.
\newblock \emph{J. Mach. Learn. Res.}, 15:\penalty0 19--26, 2011.

\bibitem[Busa-Fekete and H{\"u}llermeier(2014)]{busa2014survey}
R{\'o}bert Busa-Fekete and Eyke H{\"u}llermeier.
\newblock A survey of preference-based online learning with bandit algorithms.
\newblock In \emph{Algorithmic Learning Theory (ALT)}, pages 18--39, 2014.

\bibitem[Cesa-Bianchi and Lugosi(2006)]{Cesa-Bianchi:2006}
Nicol\`{o} Cesa-Bianchi and G\'{a}bor Lugosi.
\newblock \emph{Prediction, Learning, and Games}.
\newblock Cambridge University Press, 2006.

\bibitem[Chapelle et~al.(2012)Chapelle, Joachims, Radlinski, and
  Yue]{chapelle12:large}
Olivier Chapelle, Thorsten Joachims, Filip Radlinski, and Yisong Yue.
\newblock {Large-scale validation and analysis of interleaved search
  evaluation}.
\newblock \emph{ACM Transactions on Information Systems (TOIS)}, 30\penalty0
  (1):\penalty0 6:1--6:41, 2012.

\bibitem[Dud{\'i}k et~al.(2011)Dud{\'i}k, Hsu, Kale, Karampatziakis, Langford,
  Reyzin, and Zhang]{dudik2011efficient}
Miroslav Dud{\'i}k, Daniel Hsu, Satyen Kale, Nikos Karampatziakis, John
  Langford, Lev Reyzin, and Tong Zhang.
\newblock Efficient optimal learning for contextual bandits.
\newblock In \emph{Uncertainty in Artificial Intelligence (UAI)}, pages
  169--178, 2011.

\bibitem[Freund and Schapire(1999)]{FreundSc-GEB}
Yoav Freund and Robert~E. Schapire.
\newblock Adaptive game playing using multiplicative weights.
\newblock \emph{Games and Economic Behavior}, 29:\penalty0 79--103, 1999.

\bibitem[Gardner(1970)]{gardner:1970}
Martin Gardner.
\newblock Mathematical games: The paradox of the nontransitive dice and the
  elusive principle of indifference.
\newblock \emph{Scientific American}, 223:\penalty0 110–--114, 1970.

\bibitem[Guo et~al.(2009{\natexlab{a}})Guo, Li, and Faloutsos]{guo09:tailoring}
Fan Guo, Lei Li, and Christos Faloutsos.
\newblock {Tailoring click models to user goals}.
\newblock In \emph{Workshop on Web Search Click Data (WSCD)}, pages 88--92,
  2009{\natexlab{a}}.

\bibitem[Guo et~al.(2009{\natexlab{b}})Guo, Liu, and Wang]{guo09:efficient}
Fan Guo, Chao Liu, and Yi-Min Wang.
\newblock Efficient multiple-click models in web search.
\newblock In \emph{Proceedings of the International Conference on Web Search
  and Data Mining (WSDM)}, pages 124--131, 2009{\natexlab{b}}.

\bibitem[Hofmann et~al.(2011)Hofmann, Whiteson, and
  de~Rijke]{hofmann11:probabilistic}
Katja Hofmann, Shimon Whiteson, and Maarten de~Rijke.
\newblock A probabilistic method for inferring preferences from clicks.
\newblock In \emph{Proceedings of the International Conference on Information
  and Knowledge Management (CIKM)}, pages 249--258, 2011.

\bibitem[Hofmann et~al.(2013)Hofmann, Whiteson, and
  de~Rijke]{hofmann13:fidelity}
Katja Hofmann, Shimon Whiteson, and Maarten de~Rijke.
\newblock Fidelity, soundness, and efficiency of interleaved comparison
  methods.
\newblock \emph{ACM Transactions on Information Systems (TOIS)}, 31\penalty0
  (4), 2013.

\bibitem[Kalai and Vempala(2003)]{kalai03:fpl}
Adam Kalai and Santosh Vempala.
\newblock Efficient algorithms for the online decision problem.
\newblock In \emph{Conference on Learning Theory (COLT)}, pages 26--40, 2003.

\bibitem[Langford and Zhang(2007)]{LangfordZh07}
John Langford and Tong Zhang.
\newblock The epoch-greedy algorithm for contextual multi-armed bandits.
\newblock In \emph{Annual Conference on Neural Information Processing Systems
  (NIPS)}, pages 817--824, 2007.

\bibitem[Negahban et~al.(2012)Negahban, Oh, and Shah]{Negahban:2012}
Sahand Negahban, Sewoong Oh, and Devavrat Shah.
\newblock Iterative ranking from pair-wise comparisons.
\newblock In \emph{Annual Conference on Neural Information Processing Systems
  (NIPS)}, 2012.

\bibitem[Owen(1995)]{owen1995game}
Guillermo Owen.
\newblock \emph{Game Theory}.
\newblock Emerald Group Publishing Limited, 3rd edition, 1995.

\bibitem[Radlinski et~al.(2008)Radlinski, Kurup, and
  Joachims]{radlinski2008:how}
Filip Radlinski, Madhu Kurup, and Thorsten Joachims.
\newblock {How does clickthrough data reflect retrieval quality?}
\newblock In \emph{Proceedings of the International Conference on Information
  and Knowledge Management (CIKM)}, pages 43--52, 2008.

\bibitem[Schulze(2011)]{schulze11:monotonic}
Markus Schulze.
\newblock A new monotonic, clone-independent, reversal symmetric, and
  {C}ondorcet-consistent single-winner election method.
\newblock \emph{Social Choice and Welfare}, 36\penalty0 (2):\penalty0 267--303,
  2011.

\bibitem[Urvoy et~al.(2013)Urvoy, Clerot, {F\'eraud}, and Naamane]{Urvoy:2013}
Tanguy Urvoy, Fabrice Clerot, Raphael {F\'eraud}, and Sami Naamane.
\newblock Generic exploration and k-armed voting bandits.
\newblock In \emph{Proceedings of the International Conference on Machine
  Learning (ICML)}, pages 91--99, 2013.

\bibitem[Yue and Joachims(2009)]{Yue:2009}
Yisong Yue and Thorsten Joachims.
\newblock Interactively optimizing information retrieval systems as a dueling
  bandits problem.
\newblock In \emph{Proceedings of the International Conference on Machine
  Learning (ICML)}, pages 1201--1208, 2009.

\bibitem[Yue and Joachims(2011)]{yue2011beat}
Yisong Yue and Thorsten Joachims.
\newblock Beat the mean bandit.
\newblock In \emph{Proceedings of the International Conference on Machine
  Learning (ICML)}, pages 241--248, 2011.

\bibitem[Yue et~al.(2009)Yue, Broder, Kleinberg, and Joachims]{yue09:k-arm}
Yisong Yue, J.~Broder, R.~Kleinberg, and T.~Joachims.
\newblock The k-armed dueling bandits problem.
\newblock In \emph{Conference on Learning Theory (COLT)}, 2009.

\bibitem[Zinkevich(2003)]{Zinkevich03:online}
Martin Zinkevich.
\newblock Online convex programming and generalized infinitesimal gradient
  ascent.
\newblock In \emph{Proceedings of the International Conference on Machine
  Learning (ICML)}, pages 928--936, 2003.

\bibitem[Zoghi et~al.(2014{\natexlab{a}})Zoghi, Whiteson, de~Rijke, and
  Munos]{RCS2014}
Masrour Zoghi, Shimon Whiteson, Maarten de~Rijke, and R\'{e}mi Munos.
\newblock Relative confidence sampling for efficient on-line ranker evaluation.
\newblock In \emph{Proceedings of the International Conference on Web Search
  and Data Mining (WSDM)}, pages 73--82, 2014{\natexlab{a}}.

\bibitem[Zoghi et~al.(2014{\natexlab{b}})Zoghi, Whiteson, Munos, and
  de~Rijke]{RUCB2014}
Masrour Zoghi, Shimon Whiteson, R\'{e}mi Munos, and Maarten de~Rijke.
\newblock Relative upper confidence bound for the $k$-armed dueling bandits
  problem.
\newblock In \emph{Proceedings of the International Conference on Machine
  Learning (ICML)}, pages 10--18, 2014{\natexlab{b}}.

\bibitem[Zoghi et~al.(2015{\natexlab{a}})Zoghi, Karnin, Whiteson, and
  de~Rijke]{Copeland2015}
Masrour Zoghi, Zohar Karnin, Shimon Whiteson, and Maarten de~Rijke.
\newblock Copeland dueling bandits.
\newblock 2015{\natexlab{a}}.
\newblock arxiv:1506.00312.

\bibitem[Zoghi et~al.(2015{\natexlab{b}})Zoghi, Whiteson, and
  de~Rijke]{mergeRUCB2015}
Masrour Zoghi, Shimon Whiteson, and Maarten de~Rijke.
\newblock {MergeRUCB}: A method for large-scale online ranker evaluation.
\newblock In \emph{Proceedings of the International Conference on Web Search
  and Data Mining (WSDM)}, pages 17--26, 2015{\natexlab{b}}.

\end{thebibliography}

\appendix


\section{Failure of the Condorcet winner to exist}
\label{sec:non-condorcet}

Here, we investigate the reliability of the Condorcet assumption by replicating the experiment of \citep[Section 6.1]{RCS2014} with a small modification. As in their setting, we consider a family of $K$-armed dueling bandit problems arising from the ranker evaluation problem in IR, where the comparisons are carried out using Probabilistic Interleave~\citep{hofmann11:probabilistic} and the preferences are generated using click models simulating user behavior~\citep{guo09:tailoring,guo09:efficient}. The rankers are sampled randomly from the set of $136$ rankers provided with the MSLR dataset.\footnote{\url{http://research.microsoft.com/en-us/projects/mslr/default.aspx}} However, unlike the experiments of \cite{RCS2014}, we use an informational click model, rather than a perfect one~\citep{hofmann13:fidelity}. The former simulates the behavior of a user who is seeking general information about a broad topic, while the latter represents an idealized user, who meticulously examines every document in the retrieved list. We believe that the informational click model is more realistic and therefore use it here.

The plot in Figure \ref{fig:prob-condorcet} shows the probability with which the encountered dueling bandit problems contain Condorcet winners. As this figure demonstrates, in this setting, the occurrence of the Condorcet winner drops rapidly as the number of rankers grows.

This shows that even in this simple non-contextual example the assumption that there exists a Condorcet winner is too unreliable to be practical. Needless to say that in the contextual dueling bandit problem, where one is dealing with a potentially very large and diverse set of policies, the likelihood of one policy dominating every single other policy is even more unrealistic.

\begin{figure}[!t]

\centering
\includegraphics[width=.685\textwidth]{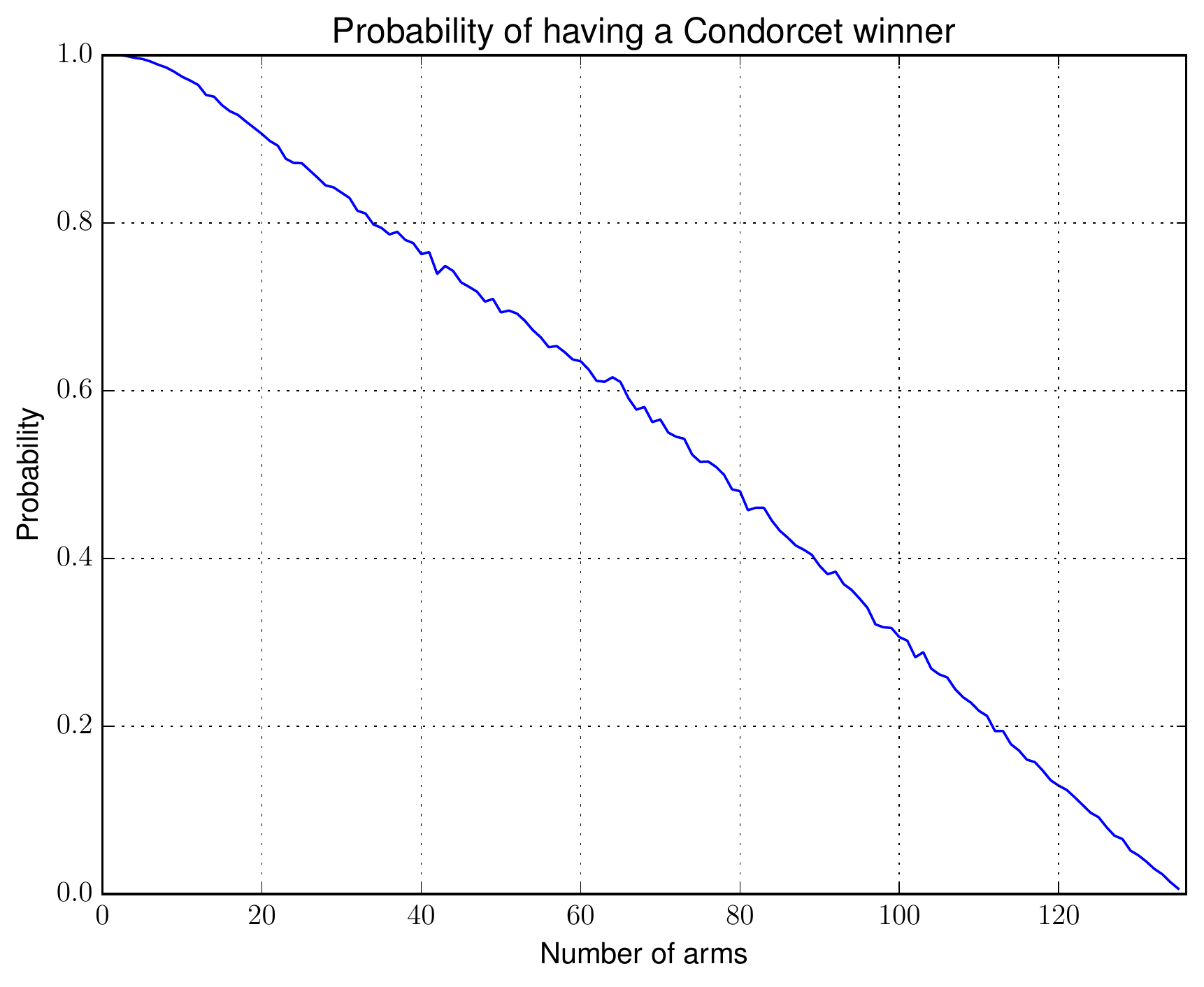}
\caption{The probability that the Condorcet assumption holds for subsets of the feature rankers in the MSLR dataset. The probability is shown as a function of the size of the subset of rankers under consideration.}
\label{fig:prob-condorcet}
\end{figure}

\section{Comparison between the Copeland and von Neumann winners}
\label{sec:copeland}

A Copeland winner is defined to be any arm that beats the largest number of other arms. It is a generalization of the Condorcet winner in the sense that if the Condorcet winner exists, it will be a Copeland winner. However, we claim that the von Neumann winner is a more natural generalization than the Copeland winner for the following two reasons: first, in the absence of a Condorcet winner, Copeland winners, both individually and as a collective, can lose to an arm that is not a Copeland winner, whereas the von Neumann winner beats or ties with every single arm; second, the set of Copeland winners can be altered by the introduction of ``clones,'' i.e., arms whose corresponding rows of the preference matrix are identical to each other.

To demonstrate this lack of stability of Copeland winners, consider any $K+3$-armed example with $K>4$, where arms $a_1, a_2$ and $a_3$ beat all other arms and the three of them are in a cycle, with $a_1$ beating $a_2$, $a_2$ beating $a_3$ and $a_3$ beating $a_1$ all with probability $1$. It is easy to see that these three arms are the only Copeland winners with Copeland score equal to $K+1$ and also form the support of the von Neumann distribution: indeed, the von Neumann distribution is simply the uniform distribution on these three arms. Now, let us consider a slight modification of this problem, where we add one more arm, called $a_0$, which is a duplicate of arm $a_1$; hence $P(0,1) = 0$ and $P(0,j) = P(1,j)$ for all $j>1$. In the following we explain what happens to the set of Copeland winners after this modification. In the presence of ties, there are three sensible definitions that one could use for the Copeland score; these definitions and the corresponding scores for the top four arms in our modified example can be found in Table \ref{tbl:cpld}. As the quantities in this table show, regardless of the definition of the Copeland score used, the set of Copeland winners for our new $K+4$-armed dueling bandit problem does not contain all of $a_0,\ldots,a_3$. Indeed, under no definition can arm $a_2$ be considered a Copeland winner.

On the other hand, arms $a_0, a_1, a_2, a_3$ still form the support of the von Neumann distribution of this modified dueling bandit problem: if we assign weights $w(0),w(1),w(2),w(3)$ to these four arms such that
\[ w(0)+w(1) = w(2) = w(3) = \frac{1}{3} \]
and sample an arm $a_i$ according to these weights, we will (on average) beat any $a_j$ with $j>3$ and tie with all $a_j$ with $j\le 3$.

We consider the lack of stability under cloning illustrated by this example to be a major drawback of the Copeland score as a measure of quality.

\begin{table}[!t]
\centering
    \caption{Copeland scores of the top arms in the duplicated example}
    \label{tbl:cpld}

    \vspace{3mm}
	\begin{tabular}{l|c|c|c}
	The Copeland score & $|\{j:\: P(i,j) > 0\}|$ & $|\{j:\: P(i,j) \ge 0\}|$ & $|\{j:\: P(i,j) > 0\}|$ \hphantom{---}
\\
    variant (for $a_i$) & & & \hphantom{---} ${} - |\{j:\: P(i,j) < 0\}|$
 \\[5pt]
	\hline
	$i=0,1$ & $K+1$ & $K+3$ & $K$ \\[5pt]
	$i=2$ 	   & $K+1$ & $K+2$ & $K-1$ \\[5pt]
	$i=3$ 	   & $K+2$ & $K+3$ & $K+1$ \\[5pt]
	\end{tabular}
\end{table}

Furthermore, as the following $5$-armed preference matrix illustrates, the von Neumann winner does not necessarily contain the Copeland winner in its support:

\begin{align*}
\pref = \begin{bmatrix}
   \;0  & \hphantom{-}0.5\hphantom{0}  & -0.5\hphantom{0}  & \hphantom{-}0.5\hphantom{0}  & -0.95 \\
 -0.5\hphantom{0}  &   \;0  & \hphantom{-}0.5\hphantom{0}  & -0.2\hphantom{0}  & \hphantom{-}0.5\hphantom{0} \\
 \hphantom{-}0.5\hphantom{0}  & -0.5\hphantom{0}  &   \;0  & -0.2\hphantom{0}  & \hphantom{-}0.5\hphantom{0} \\
 -0.5\hphantom{0}  & \hphantom{-}0.2\hphantom{0}  & \hphantom{-}0.2\hphantom{0}  &  \;0   & \hphantom{-}0.5\hphantom{0} \\
 \hphantom{-}0.95  & -0.5\hphantom{0}  & -0.5\hphantom{0}  & -0.5\hphantom{0}  &  \;0
\end{bmatrix}
\end{align*}

\begin{figure}[!t]

\centering
\includegraphics[width=.685\textwidth]{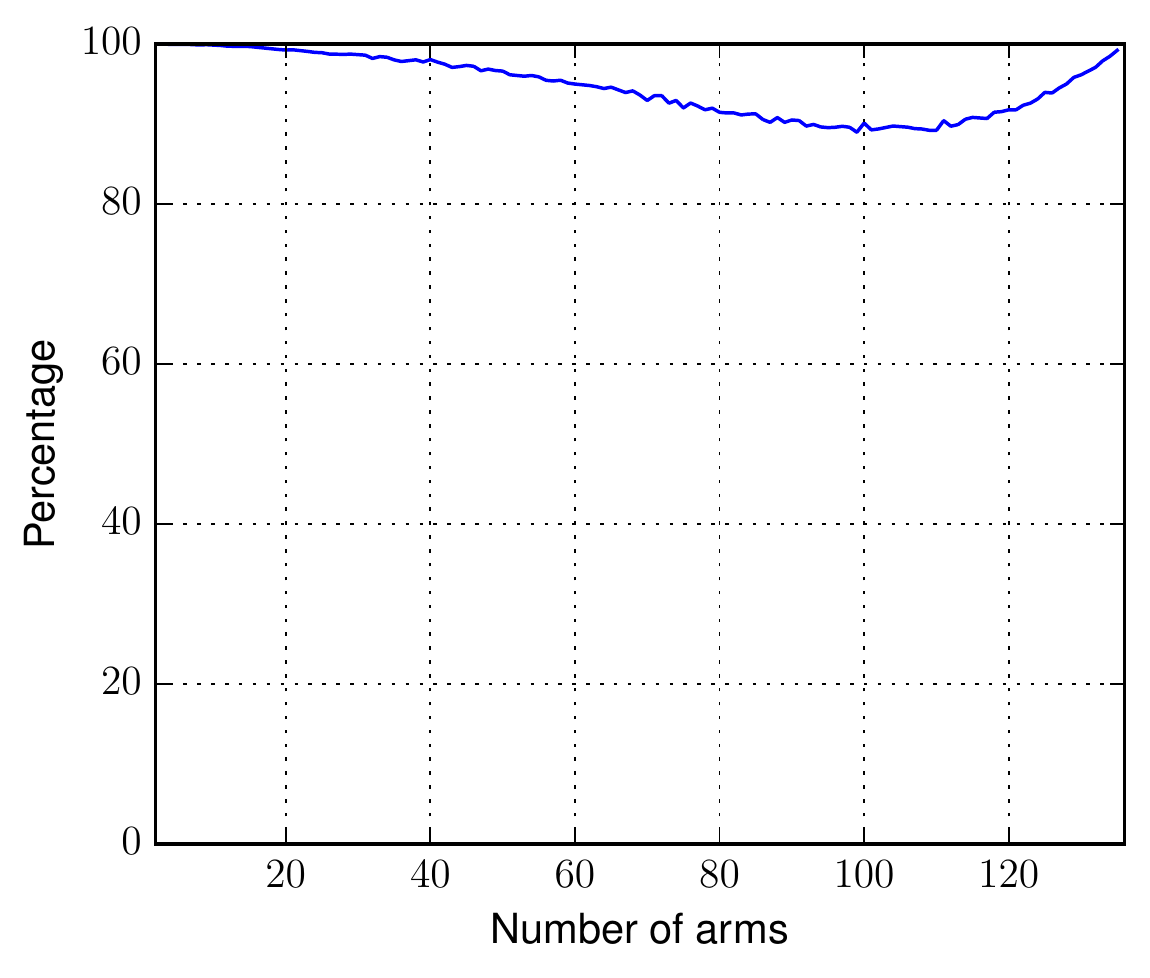}
\caption{The percentage of preference matrices of a given size, sampled from the MSLR dataset, for which all Copeland winners are contained in the support of the von Neumann winner}
\label{fig:copeland-vonNeumann}
\end{figure}

Indeed, the von Neumann winner of this matrix is the uniform distribution on the first three arms, as can be easily checked by multiplying the row vector $[1/3~1/3~1/3~0~0]$ with $\pref$, while the Copeland winner is the fourth arm, since the Copeland scores of the $5$ arms are $[2~2~2~3~1]$. Moreover, the fourth arm also happens to be both the Borda winner \citep{Urvoy:2013} and the Random Walk winner \citep{Negahban:2012,busa2014survey}. The Borda winner is the arm with the highest chance of winning a comparison against a uniformly randomly chosen opponent, i.e., the arm corresponding to the row in the preference matrix whose entries have the highest sum: in this case the Borda scores are $[0.455~~0.53~~0.53~~0.54~~0.445]$. The Random Walk winner is obtained as follows: first, we convert the ``probabilistic'' preference matrix (i.e., $\pref/2 + 0.5$) into a column-wise stochastic matrix (by dividing each column by its sum), then find the stationary distribution of the Markov chain defined by this matrix (by finding the right eigenvector of the stochastic matrix corresponding to eigenvalue $1$), and, finally, declare the arm with the highest probability under this stationary distribution to be the Random Walk winner. In this particular example, the stationary distribution is $[0.198~~0.212~~0.204~~0.217~~0.169]$ and so the Random Walk winner is the fourth arm, as mentioned before.

\begin{figure}[!t]

\centering
\includegraphics[width=.95\textwidth]{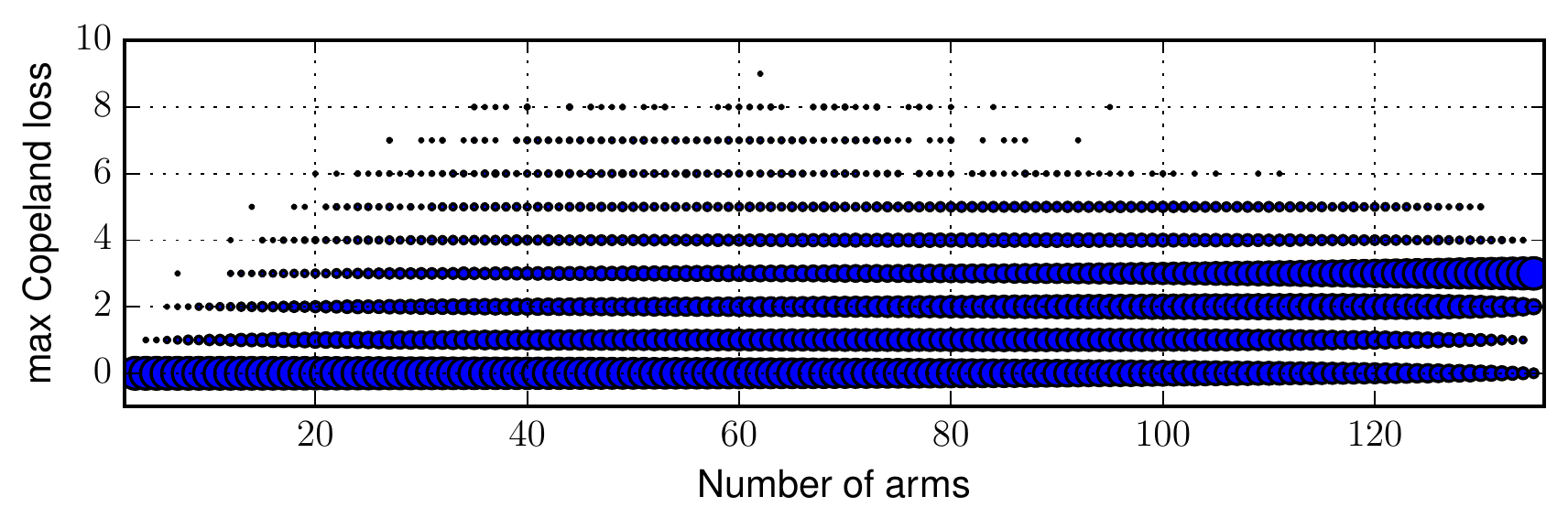}
\caption{The fractions of preference matrices of a given size (horizontal axis) with a given maximum Copeland loss among the arms in the support of their von Neumann winner (vertical axis): given $K$ and $M$, the area of the circle at coordinate $(K,M)$ of this plot is proportinal to the percentage of $K$-armed sub-matrices of the Informational MSLR preference matrix, $\pref$, for which the maximum Copeland loss of an arm in the support of the von Neumann winner of $\pref$ is equal to $M$.}
\label{fig:copeland-losses}
\end{figure}

Despite the above observations, in practice, the Copeland winner and the von Neumann winner tend to agree to a large extent. For instance, in preference matrices sampled from the MSLR dataset, as described in Appendix \ref{sec:non-condorcet}, in over 99.9\%\ of the examples, the von Neumann winner contained at least one Copeland winner. Moreover, in the overwhelming majority of the cases, all Copeland winners were assigned non-zero probability by the von Neumann winner, although the percentage of cases where this phenomenon occurs is slightly lower than the above figure and dependent on the number of arms (see Figure \ref{fig:copeland-vonNeumann}).
%
%
Based on these observations, the Copeland winner can roughly be thought of as a more restrictive notion than the von Neumann winner.
%

Furthermore, as Figure \ref{fig:copeland-losses} demonstrates, the arms in the support of the von Neumann winner tend to have high Copeland scores (or equivalently, low Copeland losses) in practice. Given this close relation between these two notions of winners, a natural question becomes whether the recent improvements made in solving the Copeland dueling bandit problem \citep{Copeland2015} can be used to speed up the task of finding the von Neumann winner.

\begin{figure}[!t]

\centering
\includegraphics[width=.685\textwidth]{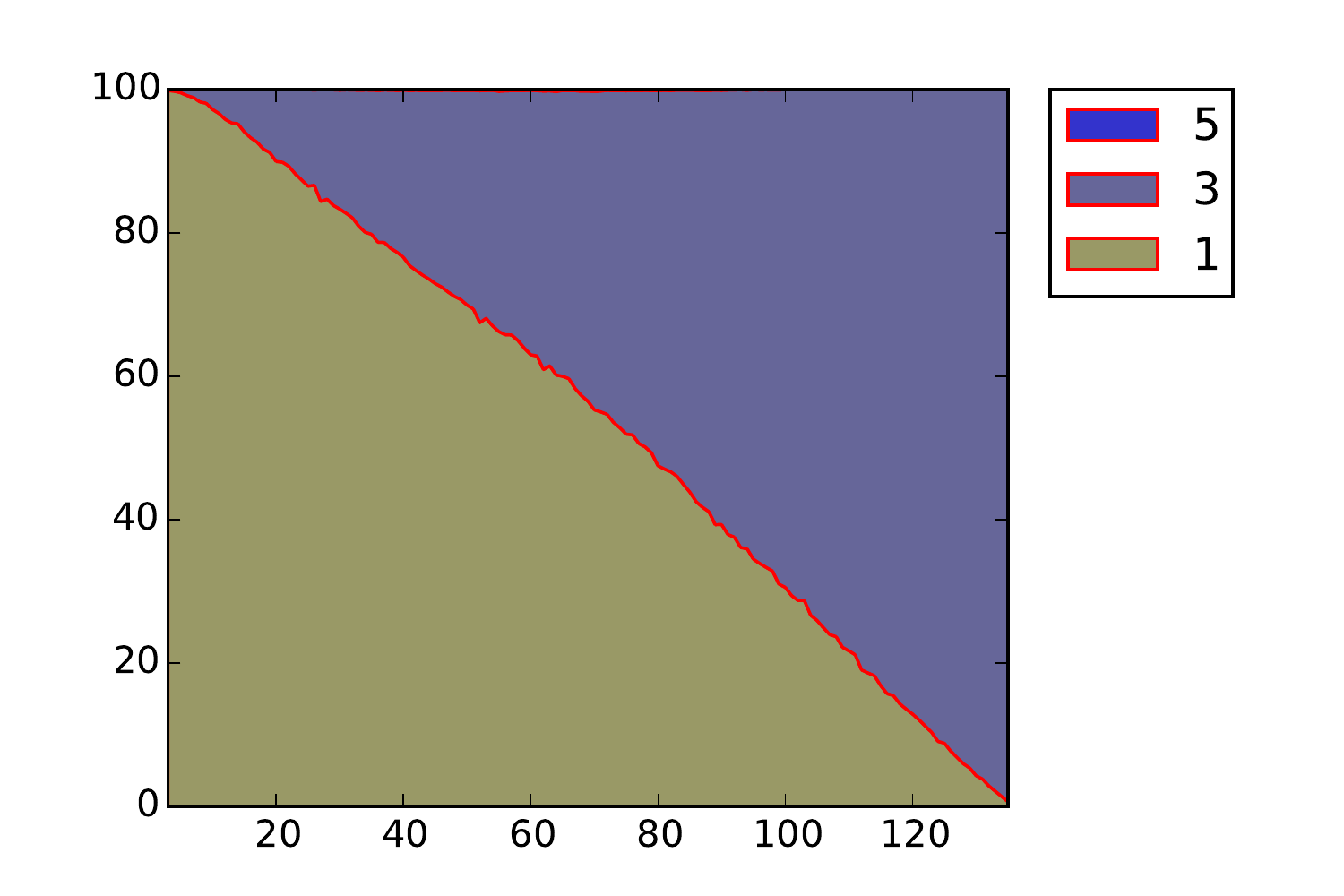}
\includegraphics[width=.685\textwidth]{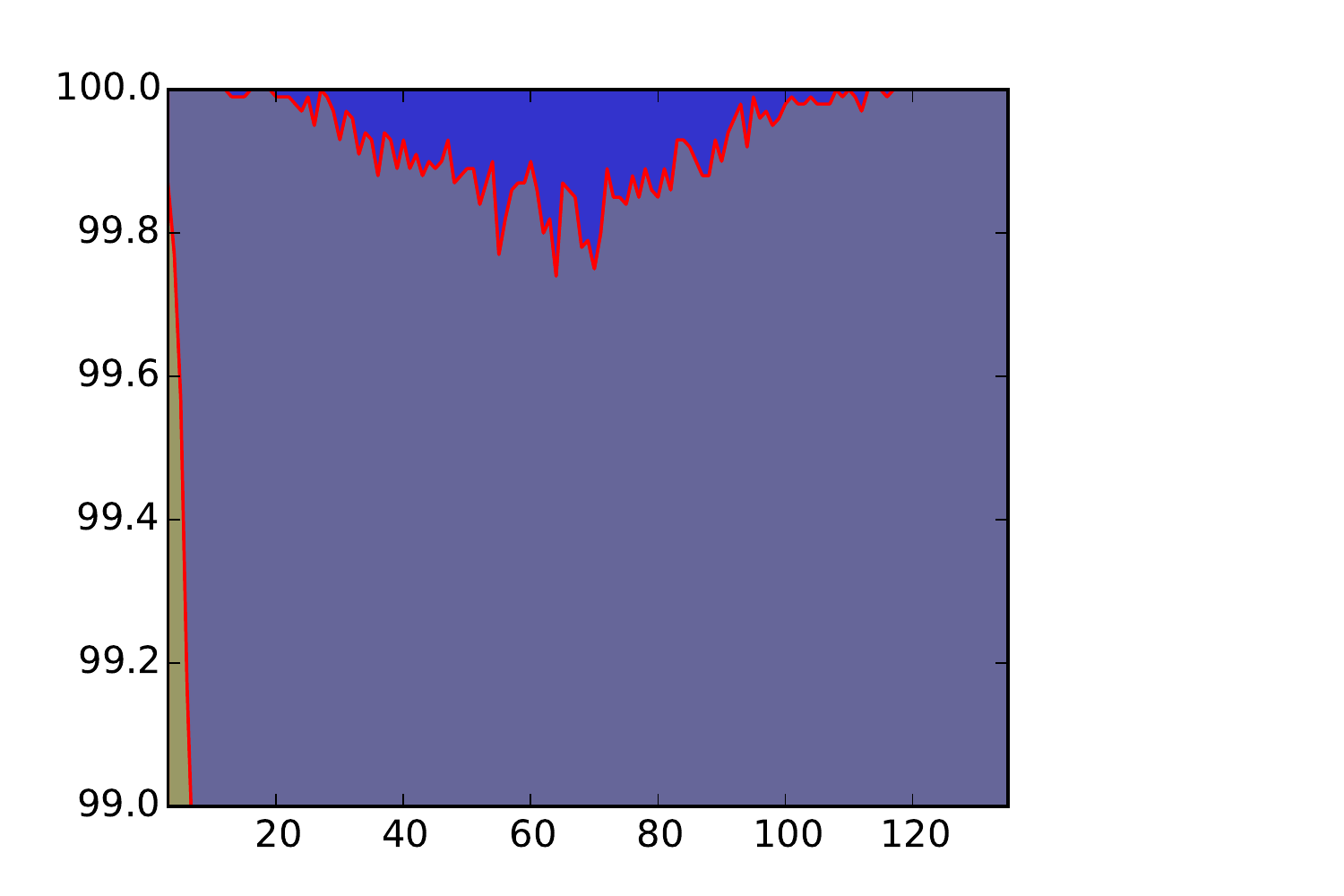}
\caption{The percentages of preference matrices of a given size with von Neumann winners of a given support size: the bottom plot is simply a zoomed version of the top one, and it was included, because preference matrices with von Neumann winners of support size 5 were very infrequent.}
\label{fig:vonNeumannSupp}
\end{figure}

Another aspect of the von Neumann winner that might be disconcerting when first encountered is the fact that it is a distribution, which in theory can put non-zero probability on all arms; however, in practice, this is very far from being the case. Indeed, among the over a million preference matrices sampled from the MSLR dataset, not a single one had a von Neumann winner that assigned non-zero probability to more than 5 arms. In fact, in the vast majority of the cases, the von Neumann winners had supports of size 1 or 3 (see Figure \ref{fig:vonNeumannSupp}). Note that fewer than 0.03\%\ of preference matrices had von Neumann winners whose support contains 5 arms.


\section{Analysis of \SparringEXP (proof of Theorem \ref{thm:sparring-exp4})}
\label{sec:pf-SparringEXP}

To fulfill the requirements of the learning model for \expfp, we also need to define rewards $r_t(a)$ for all of
the actions that were {\em not\/} chosen.
Furthermore, these rewards need to be defined {\em before\/}
each copy of the algorithm chooses its action (or, more technically, in a
manner that is conditionally independent of each copy's choice).
To this end, for every pair of actions $(a,b)$, we define
a $\{-1,+1\}$-valued random variable $R_t(a,b)$ with the
expected value $\preftab$.
Thus, $R_t(a,b)$ can be viewed as the outcome of a hypothetical duel between actions $a$ and $b$.
These values are only used for the mathematical argument, and do not
literally have to be computed.
Only the pair $(a_t,b_t)$ is actually used in a duel.

For \exprow, based on \expcol's chosen action $b_t$, we then define rewards
$r^{{\it (row)}}_t(a)=R_t(a,b_t)$ for all $a$.
And similarly,
for \expcol, based on \exprow's chosen action $a_t$, we define rewards
$r^{{\it (col)}}_t(b)= -R_t(a_t,b)$ for all $b$.
In particular, this means that \exprow\ receives, for its chosen action $a_t$, the reward
$R_t(a_t,b_t)$ (that is, the result of a duel between $a_t$ and $b_t$), while \expcol\ receives
the reward $-R_t(a_t,b_t)$ for its chosen action $b_t$.

Let us first take the point of view of \exprow.
Plugging in to \eqref{eq:exp4-g1}, we have that, with probability at least
$1-\delta/4$, for all $\pi\in\Pi$,
\[
  \sum_{t=1}^T R_t(a_t,b_t) \geq
     \sum_{t=1}^T R_t(\pi(x_t),b_t) - O(\sqrt{KT \ln(\picard/\delta)}).
\]
Further, using Azuma's lemma and union bound, we can show that, with
probability at least $1-\delta/4$, for every $\pi\in\Pi$
\[
\sum_{t=1}^T R_t(\pi(x_t),b_t) \geq
  \sum_{t=1}^T \prefeltt{\pi(x_t)}{b_t} - O(\sqrt{{K}T \ln(\picard/\delta)}).
\]
Similarly, from \expcol's perspective, with probability at least
$1-\delta/4$, for all
$\pi\in\Pi$,
\[
  -\sum_{t=1}^T R_t(a_t,b_t) \geq
     -\sum_{t=1}^T R_t(a_t,\pi(x_t)) - O(\sqrt{KT \ln(\picard/\delta)})
\]
and, by Azuma's lemma and the skew-symmetry of $\preft$, with probability at least
$1-\delta/4$, for every
$\pi\in\Pi$,
\[
-\sum_{t=1}^T R_t(a_t,\pi(x_t))
\geq
     \sum_{t=1}^T \prefeltt{{\pi(x_t)}}{a_t} - O(\sqrt{{K}T \ln(\picard/\delta)}).
\]
Combining and rearranging now yields the theorem.

\section{Analysis of \SparringFPL and \ProjectedGD}
\label{appendix-proof-perceptron}

Compared to the presentation in the body of the paper, we extend our setting from uniform exploration strategy to an arbitrary unbiased estimator $\eprefi$ of $\prefi$, i.e., to any matrix $\eprefi$ which satisfies
    $\expect{}{\eprefi \given \prefi} = \prefi$.
Our results are parameterized by two numbers, $\eprefbnd, \eprefvar$, such that
    $|\eprefiab|\leq \eprefbnd$ and $\mathtt{Var}(\eprefiab) \leq \eprefvar$
for all exploration rounds $i$ and all action pairs $(a,b)$. For uniform exploration, both upper bounds are $K^2$.

\subsection{Reduction from $\Msup$ to $\Memp$ (proof of Lemma \ref{lem:explore-first})}
\label{sec:pf-ExploreFirst}

In this subsection, we prove Lemma \ref{lem:explore-first} which reduces the optimization problem to that on the approximate matrix $\Memp$ computed from the data. As a first step, we prove \eqref{eq:f2} which relates $\Memp$ to the true preference matrix for $\Msup$.

\begin{proof}[\eqref{eq:f2}]
Let $Z_1,\ldots,Z_\numbatch$ be independent, identically distributed random
variables, each taking values in $[-R,R]$, and each having mean zero and
variance $V$.
Then according to Bernstein's inequality,
the probability that the average $A=(1/\numbatch)\sum_{i=1}^\numbatch Z_i$ exceeds
some value $s$ is at most
\[
\exp\paren{-\frac{s^2/2}
                 {\numbatch V + R s / 3}}.
\]
For an appropriate choice of $s$, this implies that, with
probability at least $1-\delta$,
\begin{equation} \label{eq:bern1}
 A \leq \frac{2 R \ln(1/\delta)}{3\numbatch} + \sqrt{\frac{2V\ln(1/\delta)}{\numbatch}}.
\end{equation}

To derive \eqref{eq:f2}, for a fixed pair of policies $\pi$ and
$\altpi$, we can let
$Z_i=\eprefielt{\pi(x_i)}{\altpi(x_i)} - \Msuppipi$
whose mean is zero and variance is at most $\eprefvar$;
further, $|Z_i|\leq 1+\eprefbnd$.
Plugging into \eqref{eq:bern1} implies that \eqref{eq:f2} holds with
\[
\empmatbnd=
 \frac{2 (1+\eprefbnd) \ln(\picard^2/\delta)}{3\numbatch}
        + \sqrt{\frac{2\eprefvar\ln(\picard^2/\delta)}{\numbatch}}
\]
with probability at least $1-\delta/\picard^2$.
By the union bound, with probability at least $1-\delta$, this will
hold simultaneously for {\em all\/} policies $\pi$ and $\altpi$.
\end{proof}

\begin{proof}[Lemma \ref{lem:explore-first}]
\eqref{eq:f2} implies that
$\trans{\wwpol} \Memp \uupol$ is within $\empmatbnd$ of
$\trans{\wwpol} \Msup \uupol$, for all probability vectors $\wwpol$ and $\uupol$.
Therefore,
\begin{eqnarray*}
\min_{\uupol\in\prsim{\picard}} \trans{\wpolhat} \Msup \uupol
&\geq&
\min_{\uupol\in\prsim{\picard}} \trans{\wpolhat} \Memp \uupol - \empmatbnd
\\
&\geq&
\max_{\wwpol\in\prsim{\picard}} \min_{\uupol\in\prsim{\picard}} \trans{\wwpol} \Memp \uupol - \empmatbnd - \algerr
\\
&\geq&
\max_{\wwpol\in\prsim{\picard}} \min_{\uupol\in\prsim{\picard}} \trans{\wwpol} \Msup \uupol - 2\empmatbnd - \algerr
\\
&=&
  - (2\empmatbnd + \algerr).
\end{eqnarray*}
\end{proof}

\subsection{Analysis of \SparringFPL}
\label{sec:pf-FPL}

To analyze \SparringFPL, we build on the provable guarantees for \FPL.

For convenience, let us recap the learning setting for \FPL.
Let $\decsp$ and $\losssp$ be subsets of $\reals^{\numbatch K}$.
On each round $t=1,\ldots,\minmaxbudget$,
the learner chooses a decision vector $\dd_t\in\decsp$, and
then receives a loss vector $\loss_t\in\losssp$.
The learner's goal is to minimize its cumulative loss
$\sum_{t=1}^\minmaxbudget \dd_t\cdot\loss_t$ relative to the best
possible loss using a fixed decision, that is,
$\min_{\dd\in\decsp} \sum_{t=1}^\minmaxbudget \dd\cdot\loss_t$.

\citet[Theorem~1.1]{kalai03:fpl} prove the following
(slightly simplified) result: assume that $D$, $R$ and $A$ are such
that for all $\dd\in\decsp$ and $\loss\in\losssp$ we have that
$\lone{\dd} \leq D$, $|\dd\cdot\loss| \leq R$ and
$\lone{\loss}\leq A$.
Also, let $\alpha=\sqrt{2D/(R A \minmaxbudget)}$.
Then, for any sequence $\loss_1,\ldots,\loss_\minmaxbudget\in\losssp$,
\begin{equation} \label{eqn:fpl1}
 \Exp{}{\sum_{t=1}^\minmaxbudget \dd_t \cdot \loss_t} \leq
        \min_{\dd\in\decsp} \sum_{t=1}^\minmaxbudget \dd \cdot \loss_t
        + 2\sqrt{2 D R A \minmaxbudget}.
\end{equation}
where the expectation is over the random choice of perturbations.
Kalai and Vempala prove this in the oblivious case when the adversary
has fixed the $\loss_t$'s ahead of time.
However, this restriction can be relaxed to allow each
$\loss_t$ to be selected adaptively in a possibly stochastic fashion
that may depend on the entire preceding history through round $t-1$,
but {\em not\/} on the perturbation $\perturb_t$ for the current round.
Using a martingale argument and Azuma's lemma
(see also \citealp[Lemma~4.1]{Cesa-Bianchi:2006}),
it can then be shown that, with probability at least $1-\delta'$,
\begin{equation} \label{eqn:fpl3}
  \frac{1}{\minmaxbudget}\sum_{t=1}^\minmaxbudget \dd_t \cdot \loss_t
   \leq
  \min_{\dd\in\decsp} \frac{1}{\minmaxbudget}\sum_{t=1}^\minmaxbudget \dd \cdot \loss_t
        + 2\sqrt{2 D R A / \minmaxbudget}
        + 2R \sqrt{2\ln(1/\delta')/\minmaxbudget }.
\end{equation}

\begin{theorem}
\label{thm:fpl}
In \SparringFPL, set parameter $\fplparm=\sqrt{2/(\eprefbnd^2 \minmaxbudget)}$.
Then with probability at least $1-\delta$, the vector $\wwbar$ returned by the algorithm satisfies
\[
  \min_{\uu\in \C} \trans{\wwbar} \bldiag \uu
  \geq
  \max_{\ww\in \C} \min_{\uu\in \C} \trans{\ww} \bldiag \uu - 2\regret
\]
where
$\regret=
2\eprefbnd \sqrt{{2 \numbatch}/{\minmaxbudget}}
+
2 \eprefbnd \sqrt{{2\ln(2/\delta)}/{\minmaxbudget}}$.
\end{theorem}

Thus, to find an $\algerr$-approximate solution, we can choose
$\minmaxbudget$ to be
$O(\eprefbnd^2/\algerr^2)(\numbatch+\ln(1/\delta))$.
This also gives a bound on the number of oracle calls (it is called twice per round).

\begin{proof}
Note that in our case, we can choose $D=\sqrt{\numbatch}$,
$A=\eprefbnd\sqrt{\numbatch}$ and
$R=\eprefbnd$.

Let
$\ubar = \frac{1}{\minmaxbudget} \sum_{t=1}^\minmaxbudget \uu_t$.
Then we have the following chain of inequalities holding with
probability at least $1 - \delta$:
\begin{eqnarray}
\min_{\uu\in\C} \max_{\ww\in\C} \trans{\ww} \bldiag \uu - \regret
 &\leq&
\max_{\ww\in\C} \trans{\ww} \bldiag \ubar - \regret
\nonumber
\\
 &\leq&
 \frac{1}{\minmaxbudget} \sum_{t=1}^\minmaxbudget  \trans{\ww_t} \bldiag \uu_t
\label{eq:i1}
\\
 &\leq&
  \min_{\uu\in \C}  \trans{\wwbar} \bldiag \uu  + \regret
\label{eq:i2}
\\
 &\leq&
  \max_{\ww\in\C} \min_{\uu\in \C}  \trans{\ww} \bldiag \uu  + \regret.
\nonumber
\end{eqnarray}
Here, \eqrefii{eq:i1}{eq:i2} follow directly from
\eqref{eqn:fpl3} applied, respectively, to \rowfpl\ and \colfpl\ with
$\delta'=\delta/2$.
Noting that
\[
\max_{\ww\in\C} \min_{\uu\in \C}  \trans{\ww} \bldiag \uu
  =
\min_{\uu\in \C} \max_{\ww\in\C}  \trans{\ww} \bldiag \uu,
\]
the theorem now follows.
\end{proof}

\subsection{Analysis of \ProjectedGD: outer loop}
\label{sec:pf-GD}

Using an analysis similar to~\cite{Zinkevich03:online},
but for a fixed learning rate,
and taking into account the errors introduced by imperfect
projections, we can show the following:
\begin{lemma}  \label{lem:proj-out1}
For the algorithm \ProjectedGD\ with
$\eta = {2}/({\eprefbnd\sqrt{\Tout}})$, we have

\[
\frac{1}{\Tout} \sum_{t=1}^\Tout \trans{\ww_t}\bldiag\uu_t
\geq
\max_{\ww\in\C} \frac{1}{\Tout} \sum_{t=1}^\Tout \trans{\ww}\bldiag\uu_t
- \algerr
\]
where
$\algerr = {2\eprefbnd}/{\sqrt{\Tout}}
               + {\eprefbnd \approxfactor}/{2}$.
\end{lemma}

\begin{proof}
For all $\ww\in\C$, we have
\begin{eqnarray}
{\lensq{\ww - \ww_{t+1}} - \lensq{\ww - \ww_t}}
&\leq&
{\lensq{\ww - \zz_{t+1}} - \lensq{\ww - \ww_t}}
        + \approxfactor \eta \eprefbnd
\label{eq:dp-g1}
\\
&=&
-2\eta \trans{(\ww - \ww_t)} \bldiag \uu_t
+ \eta^2 \lensq{\bldiag \uu_t}
        + \approxfactor \eta \eprefbnd
\label{eq:g2}
\\
&\leq&
-2\eta \trans{(\ww - \ww_t)} \bldiag \uu_t
+ \eta^2 \eprefbnd^2
        + \approxfactor \eta \eprefbnd.
\nonumber
\end{eqnarray}
Here, \eqref{eq:dp-g1} uses \eqref{eq:b1}, applied to our case where
we have $\zz_{t+1} - \ww_t = \eta \bldiag \uu_t$.
\eqref{eq:g2} follows from straightforward algebra.
Since
$\len{\ww-\ww_1}\leq 2$,
summing over $t=1,\ldots,\Tout$ yields, for all
$\ww\in\C$,
\begin{eqnarray*}
 - 4
&\leq&
\lensq{\ww-\ww_{\Tout+1}} - \lensq{\ww-\ww_1}
\\
&\leq&
-2\eta \sum_{t=1}^\Tout \trans{\ww} \bldiag \uu_t
+2\eta \sum_{t=1}^\Tout \trans{\ww_t} \bldiag \uu_t
+ \eta^2 \eprefbnd^2 \Tout
        + \approxfactor \eta \eprefbnd \Tout.
\end{eqnarray*}
Re-arranging completes the lemma.
\end{proof}

We can prove that the returned vector $\wwbar$ is an
$\algerr$-approximate maxmin solution using a technique similar
to~\cite{FreundSc-GEB}.
(Alternatively, we could use the average of the $\uu_t$'s which is an
$\algerr$-approximate minmax solution by the same proof.)

\begin{theorem}  \label{thm:proj-out2}
The vector $\wwbar$ satisfies
  $\displaystyle\min_{\uu\in \C} \trans{\wwbar} \bldiag \uu
  \geq
  \displaystyle\max_{\ww\in \C} \displaystyle\min_{\uu\in \C} \trans{\ww} \bldiag \uu - \algerr$
where $\algerr$ is as
Lemma~\ref{lem:proj-out1}.
\end{theorem}

\begin{proof}
Let
$
   \ubar = \frac{1}{\Tout} \sum_{t=1}^\Tout \uu_t
$.
Then
\begin{eqnarray*}
  \max_{\ww\in \C} \min_{\uu\in \C} \trans{\ww} \bldiag \uu
&\geq&
  \min_{\uu\in \C} \trans{\wwbar} \bldiag \uu
\\
&=&
  \min_{\uu\in \C} \frac{1}{\Tout} \sum_{t=1}^\Tout \trans{\ww_t} \bldiag \uu
\\
&\geq&
  \frac{1}{\Tout} \sum_{t=1}^\Tout \min_{\uu\in \C} \trans{\ww_t} \bldiag \uu
\\
&=&
  \frac{1}{\Tout} \sum_{t=1}^\Tout \trans{\ww_t} \bldiag \uu_t
\\
&\geq&
  \max_{\ww\in\C} \frac{1}{\Tout} \sum_{t=1}^\Tout \trans{\ww} \bldiag \uu_t - \algerr
\\
&=&
  \max_{\ww\in\C} \trans{\ww} \bldiag \ubar - \algerr
\\
&\geq&
  \min_{\uu\in\C} \max_{\ww\in\C} \trans{\ww} \bldiag \uu - \algerr.
\end{eqnarray*}
\end{proof}

\subsection{Analysis of \ProjectedGD: inner loop}
\label{sec:pf-GD-inner}

It remains to analyze the inner loop of \ProjectedGD, i.e., the approximate-projection procedure
$\approxproject(\zz,\vv_1)$.

Let $\vstar$ be the projection of $\zz$ onto the policy hull $\C$.
We can prove the following for this algorithm using
$\lensq{\vstar-\vv_t}$ as a potential function.
\begin{lemma}  \label{lem:inner-regret}
For the algorithm \approxproject\ with
$ \nu = {\len{\zz-\vv_1}}/{\sqrt{\Tin}}$ and $\regin=8\len{\zz-\vv_1}/\sqrt{\Tin}$,
we have
\[
\frac{1}{\Tin} \sum_{t=1}^\Tin F(\ss_t,\vv_t)
\geq
\frac{1}{\Tin} \sum_{t=1}^\Tin F(\ss_t,\vstar)
- \regin.
\]
\end{lemma}

\begin{proof}
We have
\begin{eqnarray}
  \lensq{\vstar - \vv_{t+1}} - \lensq{\vstar - \vv_t}
&=&
  \lensq{\vstar - \vv_t + \nu (\vv_t - \ss_t)}
          - \lensq{\vstar - \vv_t}
\nonumber
\\
&=&
 2\nu (\vstar - \vv_t)\cdot(\vv_t - \ss_t)
  + \nu^2 \lensq{\vv_t - \ss_t}
\nonumber
\\
&\leq&
 2\nu (\vstar - \vv_t)\cdot(\vv_t - \ss_t)
  + 4\nu^2
\label{eq:b4}
\\
&\leq&
 \nu(F(\ss_t,\vv_t) - F(\ss_t,\vstar)) + 4\nu^2.
\label{eq:b5}
\end{eqnarray}
\eqref{eq:b4} uses
$ \len{\vv_t-\ss_t}\leq \len{\vv_t}+\len{\ss_t} \leq 2 $.
To see \eqref{eq:b5}, note that
\begin{eqnarray*}
2(\vstar - \vv_t)\cdot(\vv_t - \ss_t)
&=&
2\vstar\cdot\vv_t - 2\lensq{\vv_t} - 2\ss_t\cdot(\vstar - \vv_t)
\\
&\leq&
{\lensq{\vstar} + \lensq{\vv_t}}
 - 2\lensq{\vv_t} - 2\ss_t\cdot(\vstar - \vv_t)
\\
&=&
\brackets{2\ss_t\cdot(\vv_t-\zz) + \lensq{\zz} - \lensq{\vv_t}}
\\
&&
-
\brackets{2\ss_t\cdot(\vstar-\zz) + \lensq{\zz} - \lensq{\vstar}}
\\
&=&
F(\ss_t,\vv_t) - F(\ss_t,\vstar).
\end{eqnarray*}
The inequality here uses the fact that, for any two vectors $\uu$ and
$\ww$, we have
$ 2\uu\cdot\ww \leq \lensq{\uu} + \lensq{\ww} $.

Also,
\begin{eqnarray}
  \len{\vstar-\vv_1}
  &\leq&
  \len{\vstar-\zz} + \len{\zz-\vv_1}
\nonumber
\\
  &=&
  \min_{\vv\in\C} \len{\vv-\zz} + \len{\zz-\vv_1}
\nonumber
\\
  &\leq&
  \len{\vv_1-\zz} + \len{\zz-\vv_1}.
\label{eq:h1}
\end{eqnarray}

Summing \eqref{eq:b5} for $t=1,\ldots,\Tin$ and combining with
\eqref{eq:h1} gives
\begin{eqnarray*}
 - 4 \lensq{\zz - \vv_1}
&\leq&
 \lensq{\vstar - \vv_{\Tin+1}} -
 \lensq{\vstar - \vv_{1}}
\\
&\leq&
 \nu \sum_{t=1}^\Tin F(\ss_t,\vv_t)
 -  \nu \sum_{t=1}^\Tin F(\ss_t,\vstar)
 + 4\nu^2 \Tin.
\end{eqnarray*}
Re-arranging and applying our choice of $\nu$ completes the lemma.
\end{proof}

Next, we show that $\vbar$ satisfies the specification given in
\eqref{eq:g3}.

\begin{theorem}
\label{thm:approx-proj}
For the algorithm \approxproject,
$  \min_{\ss\in\C} F(\ss,\vbar)
  \geq  -\regin$
with $\regin$ set as in Lemma~\ref{lem:inner-regret}.
Thus, \eqref{eq:g3} holds for $\vbar$ if we set
$\approxfactor = 8/\sqrt{\Tin}$.
\end{theorem}

\begin{proof}
Let
\[
 \sbar=\frac{1}{\Tin} \sum_{t=1}^\Tin \ss_t.
\]

Then
\begin{eqnarray}
\min_{\ss\in\C} F(\ss,\vbar)
&\geq&
\min_{\ss\in\C} \frac{1}{\Tin} \sum_{t=1}^\Tin F(\ss,\vv_t)
\label{eq:b7}
\\
&\geq&
\frac{1}{\Tin} \sum_{t=1}^\Tin \min_{\ss\in\C}  F(\ss,\vv_t)
\nonumber
\\
&=&
\frac{1}{\Tin} \sum_{t=1}^\Tin  F(\ss_t,\vv_t)
\label{eq:b8}
\\
&\geq&
\frac{1}{\Tin} \sum_{t=1}^\Tin  F(\ss_t,\vstar) - \regin
\nonumber
\\
&=&
 F(\sbar,\vstar) - \regin
\label{eq:b9}
\\
&\geq&
 \min_{\ss\in\C} F(\ss,\vstar) - \regin
\nonumber
\\
&\geq&
 \lensq{\zz-\vstar} - \regin
\geq - \regin.
\label{eq:b10}
\end{eqnarray}
\eqref{eq:b7} uses Jensen's inequality and the fact that
$F(\ss,\cdot)$ is concave for each $\ss$.
\eqref{eq:b8} follows from our choice of $\ss_t$ (which minimizes
$F(\cdot,\vv_t)$).
\eqref{eq:b9} uses linearity of $F(\cdot,\vv)$ for each $\vv$.
And \eqref{eq:b10} uses
$F(\ss,\vstar)\geq \lensq{\zz-\vstar}$
for all $\ss\in\C$, which
follows from simple Euclidean geometry and the
Pythagorean theorem.
\end{proof}

Finally, combining with Lemma~\ref{lem:proj-out1} and
Theorem~\ref{thm:proj-out2}, this shows that the overall solution $\wwbar$
will be an $\algerr$-approximate maxmin solution where
$
\algerr = {2\eprefbnd}/{\sqrt{\Tout}}
               + {4 \eprefbnd}/{\sqrt{\Tin}}
$.
Thus, we can obtain any desired value of $\algerr$ by setting
$\Tin=\Tout=\ceiling{36 \eprefbnd^2 / \algerr^2}$.
The resulting number of calls to the classification oracle will be
$\Tout+\Tin\Tout = O(\eprefbnd^4 / \algerr^4)$.
As earlier noted, compared to \SparringFPL, this bound gives a different
trade-off between $\numbatch$ and $\algerr$.
For the case that $\algerr=O(\empmatbnd)$, and with $\empmatbnd$ as in
\sectref{sec:explore-first}, this algorithm gives a better bound by a
factor of $O((\ln \picard)/K^2)$.

\end{document}